\documentclass[pdflatex,sn-mathphys-num]{sn-jnl}

\usepackage{pgfplots}
\usepackage{graphicx}%
\usepackage{multirow}%
\usepackage{amsmath,amssymb,amsfonts}%
\usepackage{amsthm}%
\usepackage{mathrsfs}%
\usepackage[title]{appendix}%
\usepackage{xcolor}%
\usepackage{textcomp}%
\usepackage{manyfoot}%
\usepackage{booktabs}%
\usepackage{algorithm}%
\usepackage{algorithmicx}%
\usepackage{algpseudocode}%
\usepackage{listings}%
\usepackage{array}

\usepackage{tabularx}
\usepackage{adjustbox}

\usepackage{amsmath,bm}
\usepackage{longtable}
\usepackage{needspace}
\usepackage{hyperref} 
\usepackage{subcaption}
\theoremstyle{thmstyleone}%
\newtheorem{theorem}{Theorem}
\theoremstyle{thmstyletwo}%

\theoremstyle{thmstylethree}%
\usepackage{algorithm, algcompatible}
\usepackage{caption}
\DeclareCaptionFormat{ruled}{\hrulefill\par#1#2#3\par\hrulefill}
\algnewcommand\INPUT{\item[\textbf{Input:}]}%
\algnewcommand\OUTPUT{\item[\textbf{Output:}]}%

\usepackage{amsmath,amsfonts,bm}

\def\eqref#1{equation~\ref{#1}}

\def\1{\bm{1}}

\def\vh{{\bm{h}}}

\def\vs{{\bm{s}}}

\def\vu{{\bm{u}}}

\def\vx{{\bm{x}}}

\def\vz{{\bm{z}}}

\def\mG{{\bm{G}}}

\def\mS{{\bm{S}}}
\def\mT{{\bm{T}}}

\DeclareMathAlphabet{\mathsfit}{\encodingdefault}{\sfdefault}{m}{sl}
\SetMathAlphabet{\mathsfit}{bold}{\encodingdefault}{\sfdefault}{bx}{n}

\def\gG{{\mathcal{G}}}

\def\gX{{\mathcal{X}}}
\def\gY{{\mathcal{Y}}}
\def\gZ{{\mathcal{Z}}}

\newcommand{\E}{\mathbb{E}}

\begin{document}

\title[OSYN]{Using Synthetic Data to estimate the True Error is theoretically and practically doable}


\author[1,2]{\fnm{Hai} \sur{Hoang Thanh}}\email{hthai@tueba.edu.vn}
\equalcont{These authors contributed equally to this work.}

\author[1]{\fnm{Duy-Tung} \sur{Nguyen}}\email{ tungnd@soict.hust.edu.vn}
\equalcont{These authors contributed equally to this work.}

\author[3]{\fnm{Hung} \sur{The Tran}}\email{tran.thehung1705@gmail.com}

\author*[1]{\fnm{Khoat} \sur{Than}}\email{ khoattq@soict.hust.edu.vn}

\affil[1]{\orgname{Hanoi University of Science and Technology}, \city{Hanoi}, \country{Vietnam}}

\affil[2]{\orgname{Thai Nguyen University of Economics and Business Administration},  \state{Thainguyen}, \country{Vietnam}}

\affil[3]{\orgname{Vietnam Posts and Telecommunications Group}, \city{Hanoi}, \country{Vietnam}}



\abstract{Accurately evaluating model performance is crucial for deploying machine learning systems in real-world applications. Traditional methods often require a sufficiently large labeled test set to ensure a reliable evaluation. However, in many contexts, a large labeled dataset is \textit{costly and labor-intensive}. Therefore, we sometimes have to do evaluation by a few labeled samples, which is theoretically \textit{challenging}. Recent advances in generative models offer a promising alternative by enabling the synthesis of high-quality data. In this work, we make a systematic investigation about the use of synthetic data to estimate the test error of a trained model under limited labeled data conditions. To this end, we develop novel generalization bounds that take synthetic data into account. Those bounds suggest novel ways to optimize synthetic samples for evaluation and theoretically reveal the significant role of the generator's quality. Inspired by those bounds, we propose a theoretically grounded method to generate optimized synthetic data for model evaluation. Experimental results on simulation and tabular datasets demonstrate that,  compared to existing baselines, our method achieves accurate and more reliable estimates of the test error. }




\maketitle
\section{Introduction}
\label{Intro}

Over the past decade, advancements in Artificial Intelligence and Machine Learning (ML) have reached remarkable milestones, impacting numerous fields and sectors. There are increasing demands for using ML models in practical applications. To achieve sustainable and reliable ML applications, one fundamental step is to \textit{accurately evaluate the accuracy of the trained model} before deployment. Traditional methods \cite{raschka2020modelevaluation,bates2024cross} often require a large test set of labeled data to obtain a reliable evaluation. Unfortunately, in many situations, one may not have a rich set of labeled data for evaluation.

 Situations with limited labeled data are popular in many fields due to several reasons, such as costly data collection, rare phenomena, or privacy restrictions. For example, in the medical sector, there are limited patient records with rare diseases or small sample sizes in drug trials. In climate studies, predicting rare natural disasters like tsunamis, earthquakes has limited historical data or evaluating models for long-term climate change is challenging due to slow data collection. Together with the availability of recent large pre-trained models, tranfer learning is increasingly popular and widely used in various fields. However, the evaluation of these models often faces small dataset contraints due to domain specificity or labeling costs \citep{ayana2024multistage, laurer2024less, vaghefi2022deep}.

 Estimating model performance when test data is scarce is challenging. The main challenge is that \textit{insufficient amounts of test data can cause inaccurate and highly varying estimates about performance}. To avoid the overfitting-to-the-testing-data issue, there is an approach to estimate model performance without a test set \citep{corneanu2020computing}. Another approach is to predict trends in the quality of neural networks in scenarios where one may not have access to training data or test data  \citep{martin2021predicting}. Evaluating models with unlabeled data \citep{peng2023came, peng2024energy, johnson2023inconsistency, deng2022strong, yu2022predicting} is also of interest,  due to the prevalence of unlabeled data in many real-world applications.

 We are interested in exploiting synthetic data, generated by some generative models, to evaluate the quality of a trained model under limited labeled data conditions. From the practical perspective, we focus on the following questions:
\textit{\begin{enumerate}
    \item[(\textbf{Q1})] Can synthetic data combined with a few labeled samples enable an accurate estimate of the true error of a trained model?
    \item[(\textbf{Q2})] How to generate good synthetic samples for evaluation?
    \item[(\textbf{Q3})] How can the quality of a generator, used to produce synthetic data, affect the evaluation result?
\end{enumerate}}
 The use of synthetic data for evaluation is both practical and cost-efficient. Recent breakthroughs in generative models (e.g., ChatGPT, Gemini, Sora, ...) have made it possible to produce high-quality data, which are hardly distinguishable from real ones, with an unlimited amount. This has prompted the exploration of using synthesized data to aid model assessment \cite{zhang2022predicting,van2023can}. Even with imperfect generators, the noise of synthetic estimates can be significantly lower than the real estimates \citep{van2023can}. \citet{zhang2022predicting} empirically found that synthetic data allows us to well estimate the test error, thanks to a closer distance from the test set than the training set. While providing some promising answers to question (Q1), those studies do not provide reasonable answers to questions (Q2) and (Q3). Furthermore, the methods proposed by  \cite{zhang2022predicting,van2023can} are mostly empirical, without a theoretically-justified support. This is a major limitation and requires an entirely different approach to the three questions mentioned above.

 Our contributions in this work are as follows: 
\begin{itemize}
    \item \textit{Theory:} We systematically investigate the use of a few real labeled samples and synthetic samples to estimate the true error of a trained ML model. To this end, we first develop novel bounds on the true error, which can encode both synthetic and real data distributions. Those bounds consider the connection between the behaviors of a model at different local areas of the data space and generalization. Our bounds indicate that a synthetic distribution that is closer to the real one will be better.
    
    More importantly, those bounds suggest novel ways to choose synthetic samples, generated from a chosen generator, to accurately estimate the true error of a given trained model. 
    \item \textit{Method:} Inspired by those theoretical bounds, we propose a simple method to find a good \textit{synthetic test data} from a generator to evaluate the performance of a trained model. Therefore, our methods and theory provide reasonable answers to the three questions mentioned above.
    \item \textit{Empirical validation:} We empirically evaluate our proposed method on some simulation and tabular datasets, and find that our method can estimate the performance of a model better than the baselines.
\end{itemize}


\section{Related Work}

\textbf{Model Evaluation.} The prevailing testing approach in machine learning is to evaluate performance by using hold-out test sets. This traditional way can face a problem of overfitting to the testing set, when one keeps modifying the models until it works well on this dataset. Hence, \citet{corneanu2020computing} proposed an algorithm to compute the accuracy of a neural network without the need for the collection of any testing data.  However, this method requires training data and training models multiple times, which is highly computational and storage demanding, or may be infeasible in case the training data is inaccessible. 

 When evaluating a model using a test set, the dataset size is an important concern. Insufficient data may lead to noisy estimated results. One direction to address this problem is to increase the size of datasets by collecting more real data \citep{koh2021wilds}.  The drawbacks of this solution are the high labor demand, high cost, and dependence on specific tasks. Another approach for efficient model evaluation is active testing. The goal of active testing for evaluating a single predictive model is to estimate the average test loss by actively selecting data instances to be labeled, ensuring model evaluation is sample-efficient and unbiased \cite{kossen21,kossen2022active}. Active testing is also employed in model comparison and selection, with the goal of identifying the best-performing model from a set of candidates using only a limited number of labeled instances \cite{sawade2012active, hara2024active}. In this study, we consider the problem of model evaluation in a setting different from active testing,  where the test set is \textit{given, fixed}, and of limited size. Our setting focuses on estimating the expected loss of a fixed model using a small i.i.d. labeled sample from the target distribution, and it is related to several existing lines of work. Small-sample statistical methods (e.g., exact tests, advanced bootstrap, Bayesian approaches) provide general-purpose tools for inference under limited data \cite{mehta2012exact, rubin1981bayesian}. However, unlike our setting, exact tests mainly address hypothesis testing in small-sample regimes rather than direct estimation of the expected risk.  Cross-validation techniques (such as leave-one-out, stratified, or nested CV) are widely used for generalization error estimation, yet they assume retraining across folds and are less applicable when the model is already fixed \cite{kohavi1995study}. Uncertainty quantification methods (e.g., conformal prediction  \cite{shafer2008tutorial}, Bayesian model averaging \cite{hoeting1999bayesian}) also aim to characterize uncertainty, but focus either on per-sample prediction sets or on averaging across multiple models, whereas our problem concerns the true loss of a single, fixed model. 

 \textbf{Synthetic Data.} Advancements in deep generative models have promoted the widespread use of synthetic data for various purposes \citep{van2023beyond}, such as privacy \citep{jordon2018pate, assefa2020generating}, fairness \citep{xu2019achieving, van2021decaf}, and improving the quality of downstream models \citep{antoniou2017data, dina2022effect, das2022conditional, bing2022conditional,nguyen2025synthetic}.  Many studies have shown the remarkable effectiveness of training models with high-quality synthetic data in scenarios where data collection or distribution is challenging. In long-tailed learning, generative models can be used to enlarge data, resulting in a more balanced distribution \citep{shin2023fill}. \citet{yuan2023real} showed that a model trained exclusively on synthetic data can achieve strong performance on the test set. \citet{nguyen2025synthetic} showed the benefits of synthetic data for training few-shot models.

 \textbf{Synthetic Data for Model Evaluation.}
This study provides a different use of synthetic data: aiding the evaluation of machine learning models. We are not the first to tackle this. \citet{zhang2022predicting} used synthetic data generated by GANs trained on the same training set to predict the test error of a model. \citet{chen2024autoeval} constructed an auto-evaluation procedure that combines a small amount of human data and a large amount of synthetic data to get better estimates of performance. 
Synthetic data can be used to do model selection, i.e. rank a set of trained models \cite{shoshan2023synthetic}. \citet{van2023can} evaluated model performance on a certain small subgroup
by using conditional GANs to generate more data to balance the subgroup samples. A common limitation of those studies is the lack of theoretical guarantees or support for their methods. In contrast, we present a systematic theoretical analysis about the benifits of synthetic data for model evaluation. Our proposed method for evaluation is thus theoretically justified and supported.

\section{Model Evaluation with Limited Data}

We first describe the problem setting in details. Then we present some theoretical bounds on the true error of a trained model, incorporating synthetic data or a distribution. Some important properties and implications will be discussed.

\subsection{Preliminaries}

\textbf{Notations:} A bold character (e.g., $\vz$) often denotes a vector, while a bold big symbol (e.g., $\mS$) often denotes a matrix or set.  $| \mS |$ denotes the size/cardinality of a set $\mS$, and $[K]$ denotes the set $\{1, ..., K\}$ for a given integer $K \ge 1$. 

 We will work with a model $\vh$ which maps from an input space $\gX$ to an output space $\gY$. For each input $\vx \in \gX$, the prediction $\vh(\vx)$ by the model can be corrected or incorrected. In practice, we often use a \textit{loss (or measuring) function} $\ell(\vh,\vz)$ to measure the quality of prediction $\vh(\vx)$ by the model at an instance $\vz =(\vx,y) \in \gX \times \gY$. The overall quality of $\vh$ is measured by $F(P, \vh) = \underset{\vz \sim P}{\E}[\ell(\vh,\vz)]$, according to the data distribution $P$. Sometimes, the \textit{empirical} version $F(\mS, \vh) =  \dfrac{1}{|\mS|} \sum\limits_{\vz \in \mS} \ell(\vh,\vz)$ is used to see the quality on a sample $\mS$ of finite size. 

 \textbf{Problem setting:}
Given a model $\vh$ which has been trained from some samples from the real distribution $P_0$, we want to \textit{know} the overall quality of $\vh$, e.g. through its \textit{expected (or true) error} $F(P_0, \vh)$. If we have a large test set $\mS$, then we can confidently use $F(\mS, \vh)$ to approximate $F(P_0, \vh)$, which incurs negligible error owing to Hoeffding's inequality. However, we focus on the special setting where the test set is small, but we can access a generator that can produce (possibly unlimited) synthetic samples. We are interested in \textit{How to effectively use a generator $\gG$ and a few labeled instances from $P_0$ to accurately estimate the true error of $\vh$?}

\subsection{Theoretical bounds for the true error}
To see the true error of a model $\vh$, the traditional way often focuses on $F(P_0, \vh)$, which measures the (true) average error over the whole data population. A smaller value suggests a better model. In the following, we present some (lower and upper) bounds on this quantity, which reveal some novel insights.

\subsubsection{Lower bound}

Let $\Gamma(\gZ) := \bigcup\limits_{i=1}^{K } \gZ_i$ be a partition of $\gZ$ into $K $ disjoint nonempty subsets. Denote $\mS_i = \mS \cap \gZ_i $, and $n_i = | \mS_i |$ as the number of samples  falling into $\gZ_i$, meaning that $n = \sum\limits_{i=1}^K n_i$. Denote $\mT_S = \{ i \in [K ] : \mS \cap \gZ_i \ne \emptyset \}$.  Similarly, for a synthetic set $\mG$, denote $\mG_i = \mG \cap \gZ_i $, and $g_i = | \mG_i |$ as the number of samples  falling into $\gZ_i$, meaning that $g = \sum\limits_{i=1}^K g_i$. We also often use $\epsilon(\mS_i, \mG_i)  = \dfrac{1}{| \mS_i |.| \mG_i |} \sum\limits_{\vs \in \mS_i, \vu \in \mG_i} | \ell(\vh,\vs) -  \ell(\vh,\vu) |$ to denote the average loss difference for  two sets $(\mG_i, \mS_i)$ of finite size.

\begin{theorem} \label{low-bound}
Consider a trained model $\vh$, a dataset $\mS$ containing $n$ i.i.d. samples from a real distribution $P_0$, and $\mG$ with $g$ synthetic samples. Denote $\epsilon_h(\mG, \mS) = \sum_{i \in \mT_S}  \frac{ g_i}{g}  \epsilon(\mG_i, \mS_i)$, $C_h = \sup_{\vz} \ell(\vh, \vz)$, $\hat{a} = \max_{i \in [K]}  a_i$, and $\beta= 2\sum_{i \in [K]}  p_i a_i^2$, where $a_i = {\E}_{\vz \sim P_0}[\ell(\vh,\vz) | \vz \in \gZ_i]$ and $p_i = P_0(\gZ_i)$ for each $i$. Assume that $\beta>0$, $F(\mG, \vh) \ge  \epsilon_h(\mG, \mS) + B$, and $(g_1, ..., g_K)$ follows the mutinomial distribution with parameters $g$ and $(p_1, ..., p_K)$. Then for any $\delta_1 > \exp(- {0.5 g\beta}/{\hat{a}^2})$ and $\delta_2>0$, denoting $B =  C_h\sqrt{-(0.5\ln\delta_2) \sum_{i \in \mT_S}(g_i/g)^2}$ and $D = -\frac{\hat{a}}{g} \ln\delta_1$, the following holds  with probability at least $1 - \delta_1 - \delta_2$ (over the sampling of both $\mS$ and $(g_1, ..., g_K)$): 
\begin{equation}
  \label{lower-ine}
   F(P_0, \vh) \ge \left(\sqrt{F(\mG, \vh) - \epsilon_h(\mG, \mS) - B + D} - \sqrt{D}\right)^2
\end{equation}
\end{theorem}

 This theorem, which is proven in Appendix~\ref{Proof-main-results}, provides a lower bound for the true error of model $\vh$. The assumption about $\beta$ means that model $\vh$ is \textit{imperfect}, since there exists an $a_i >0$. This assumption is practical, as many models trained in practice are imperfect. Furthermore, the assumption of $F(\mG, \vh) \ge  \epsilon_h(\mG, \mS) + B$ means that model $\vh$ can be really bad for some synthetic dataset $\mG$. This is also practical, since both model $\vh$ and the generator that generates synthesized samples are often imperfect. 

 Many \textbf{interesting implications} can be derived from Theorem~\ref{low-bound}:
\begin{itemize}
    \item The true error of $\vh$ can be bounded by using the error $F(\mG, \vh)$ on a synthetic data $\mG$ and the sensitivity $\epsilon_h$ of the loss around the real samples $\mS$. A higher error on the synthetic dataset suggests a higher true error. 
    \item The true error will be high, almost surely, if we can find a large synthetic set so that the bound (\ref{lower-ine}) is high. This happens when both high $F(\mG, \vh)$ and small sensitivity $\epsilon_h$ and small constant $B$ appear. 
    \item This theorem requires a mild condition on the counts $(g_1, ..., g_K)$, but does not have any more specific requirement on $\mG$. Those properties are particularly important for finding a suitable synthetic set for evaluation.
    \item \textit{Bound (\ref{lower-ine}) does not explicitly depend on the size $n$ of the real samples.} This is surprising and truly significant for evaluation, especially for settings with limited labeled data. Note that $\mT_S$ slightly depends on $n$. An increase in $n$ probably leads to an increase in $|\mT_S|$. 
    \item Bound (\ref{lower-ine}) \textit{suggests an amenable way to optimize a synthetic dataset for estimating the true error of $\vh$}. Indeed, one can maximize the lower bound w.r.t. $\mG$. In other words, we can generate a large amount $g$ of synthetic samples and then choose the one that maximizes the lower bound. This is doable and practical.
    \item Last but not least, \textit{the lower bound (\ref{lower-ine}) can be computed exactly, once we have $\mG$}. This is one of the utmost important properties of our bound, which has never appeared in the literature.
\end{itemize}

\subsubsection{Asymtotic case}

Although our bound (\ref{lower-ine}) has many intriguing properties, we may not easily see the role of the generator's quality. To see it, we need to investigate the asymptotic case of our bound (\ref{lower-ine}) as $n \rightarrow \infty$ and $g \rightarrow \infty$.

\begin{theorem} \label{Asymtotic-theorem}
Given notations and assumptions as in Theorem~\ref{low-bound}, denote $d(P_g, P_0 | \gZ_i)  = \sum_{\vs \in \mG_i}  \E_{\vz \sim  P_g, \vs \sim  P_0} [| \ell(\vh,\vz) -  \ell(\vh,\vs) |  : \vz, \vs \in \gZ_i]$ and $p_i^g = P_g(\gZ_i)$, for $i \in [K]$, where $P_g$ be the synthetic distribution induced from the generator that generates $\mG$.  Then
\begin{eqnarray}
\label{Asymtotic-theorem-eq-01}
F(P_0, \vh) &\ge& F(P_g, \vh) - \sum_{i \in [K]} p_i d(P_g, P_0 | \gZ_i) \\
\label{Asymtotic-theorem-eq-02}
\sum_{i \in [K]} p_i^g a_i &\le& F(P_g, \vh) + \sum_{i \in [K]} p_i^g d(P_g, P_0 | \gZ_i)
\end{eqnarray}
\end{theorem}

 There is a close relation between bounds (\ref{lower-ine}) and (\ref{Asymtotic-theorem-eq-01}). Indeed, as $n \rightarrow \infty$ and $g \rightarrow \infty$, observe that $\frac{g_i}{g} \rightarrow p_i$, $D \rightarrow 0$, $F(\mG, \vh) \rightarrow F(P_g, \vh)$, and $\epsilon_h(\mG_i, \mS_i) \rightarrow d(P_g, P_0 | \gZ_i)$. They thus suggest that bound (\ref{lower-ine}) will go to (\ref{Asymtotic-theorem-eq-01}) - $C_h\sqrt{-(0.5\ln\delta_2) \sum_{i \in [K]} p_i^2}$ as $n \rightarrow \infty$ and $g \rightarrow \infty$. Furthermore, $\sum_{i \in [K]} p_i^2 \rightarrow 0$ as we partition the data space into infinitely many areas with diameter approaching to 0. Therefore, bound (\ref{lower-ine}) will converge to (\ref{Asymtotic-theorem-eq-01}).

Theorem \ref{Asymtotic-theorem} provides both upper and lower bounds for the true error of $\vh$. It suggests various implications:
\begin{itemize}
    \item While $F(P_0, \vh)$ represents the \textit{micro-average error} at the individual prediction, quantity $\sum_{i} p_i^g a_i$ can be seen as the \textit{macro-average error} that summarizes the errors ($a_i$) at different small areas. As the size of every $\gZ_i$ goes to 0, meaning $K \rightarrow \infty$, those two average errors will be equal.
    \item The true error of $\vh$ should not be far from its error w.r.t. $P_g$. The error gap is at most $\sum_{i} p_i d(P_g, P_0 | \gZ_i)$, which represents the loss-based distance between the two distributions. \textit{Those observations and bounds (\ref{Asymtotic-theorem-eq-01},\ref{Asymtotic-theorem-eq-02}) imply that the quality of the generator plays an important role.}
    \item As $P_g$ goes to $P_0$, meaning a better generator, those bounds will be tighter. This suggests that the quality of the generator is shown to be important both in theory and practice.
    \item Last but not least, this theorem also suggests that our focus on optimizing $F(\mG, \vh) - \epsilon_h(\mG, \mS)$ in the next section is theoretically justified.
\end{itemize}

\section{OSYN: Method for Estimating the True Error with Synthetic Data} \label{osyn_method}

Theorem \ref{low-bound} suggests a method to help estimate the true error of the model: find a set of synthetic points that maximize the R.H.S of  (\ref{lower-ine}). Intuitively, we can generate a large number of points from some generator $\gG$, and then select the ones that maximize the  R.H.S of  (\ref{lower-ine}). The generation and selection of points will be carried out over multiple consecutive iterations, so that with each iteration, the lower bound of the true error gradually increases. After each step, the optimal generated set is stored and combined with the generated elements in the next step to filter out new optimal points. This process helps to progressively find points that increasingly maximize the objective function over iterations. Algorithm~\ref{alg-2} summarizes the main steps.

\begin{algorithm}[t]
	\caption{OSYN: Optimizing synthetic data for evaluation}
	\label{alg-2}
	\begin{algorithmic}[1]
		\REQUIRE  Trained model $\vh$; small labeled set $\mS$; Generator $\gG$; number of iterations $T$; adjusted number of synthetic points $g^*, g_1^*, \ldots, g_K^*$; 
		\STATE \textbf{Initialization:} $\mG_{i,opt}^0 = \emptyset;$ $loss_i = \emptyset, \mG^t_{i, search} = \emptyset; a_i = 0, \forall i \in [K]; \hat{a} = 0$
		\STATE 	For each iteration $t:$
			\STATE \ \ \ Generate synthetic set $\mG^t$; 				\STATE \ \ \ Find current synthetic points of each area $\mG_i^t = \mG^t \cap \mathcal{Z}_i$.
			\STATE \ \ \ For each area $i$:
			
			\STATE \ \ \  \ \ \ Update $loss_i  = loss_i\cup \{\ell(\vh, \vu), \vu \in \mG_i^t\}$
			\STATE \ \ \  \ \ \ Update $a_i, \hat{a}$
			\STATE \ \ \  \ \ \ Create $\mG_{i, search}^t = \mG_{i, opt}^{t-1}\cup \mG_i^t$ 
		    \STATE \ \ \  \ \ \  Find $g_i^*$ points $\mG_{i, opt}^t \subseteq \mG_{i, search}^t$ that maximize $Target_i$ according to $loss_i$
		\OUTPUT  $\hat{a}, \mG_{opt}^T = \bigcup\limits_{i=1}^{K}\mG_{i, opt}^T$
	\end{algorithmic}
\end{algorithm}

 We now describe our method in detail. All quantities $(g_i, \mG_i, \gZ_i, \mT_S)$  depend on the partition $\Gamma$ of instance space $\gZ$. So, for simplicity, we fix the partition $\Gamma$ which divides $\gZ$ into small areas, each of which uses a sample in $\mS$ to be the center. This means $K=|\mS|$.

 To find a good set of synthetic points that efficiently supports model evaluation, we first need to determine the number of generated elements per area $g_i$, and then identify specific optimal generated points within each area. According to Theorem~\ref{low-bound}, the number of generated elements $g_i$ must follow multinomial distribution with parameters $g$ and $(p_1, ...,p_K)$. To obtain a good synthetic set, the R.H.S. of (\ref{lower-ine}) needs to be high, which implies that the value of $B$ should be small. It is easy to see that the term $\sum_{i \in \mT_S}(g_i/g)^2$ in $B$ is minimized when $\frac{g_i}{g} = \frac{1}{K}$. Therefore, after sampling $g_i$'s from the multinomial distribution, we adjust $g_i$ such that $\left|\frac{g_i}{g} - \frac{1}{K} \right| \leq \frac{b}{K}$, where $b$ is a hyperparameter. The goal of this adjustment is to find good $g_i$ values that both satisfy the condition about $(g_1,...,g_K)$ and help maximize the R.H.S. of (\ref{lower-ine}). Let $g^* = g_1^* +\cdots + g_K^*$ be the adjusted total number of synthetic samples. Because the real distribution $P_0$ is unknown, we approximate $(p_1, ...,p_K)$ by the synthetic distribution $P_\gG$. This approximation is reasonable and acceptable because it is well-known that the distribution induced by a high-quality generator can closely approximate the true data distribution. Therefore, $(p_1, ...,p_K)$ can be easily estimated by generating numerous synthetic points and computing the proportion of points that fall into each area.

 With fixed partition and chosen $g_i^*$, optimizing the R.H.S of (\ref{lower-ine}) reduces to following optimization problem
$$\underset{\mG}{\max}\left(F(\mG, \vh) - \sum\limits_{i \in \mT_S}\dfrac{g_i^*}{g^*}\epsilon(\mG_i, \mS_i)\right)= \underset{\mG}{\max}\dfrac{1}{g^*}\sum\limits_{i \in \mT_S}\sum\limits_{\vu \in \mG_i}\left[\ell(\vh, \vu) - |\ell(\vh, \vu) - \ell(\vh, \vs_i)|\right]$$
For each area $\gZ_i$, we seek each $\mG_{i, opt}$  independently, where $\mG_{i, opt}$ is the set of points from $G^t_{i,\text{search}}$ 
which maximizes ${Target}_i = \ell(\vh, \vu) - |\ell(\vh, \vu) - \ell(\vh, \vs_i)|$. More precisely, for a chosen number $T$ of iterations, we 1) iteratively generate synthetic points and divide them into each area; 2) update values of $a_i, \hat{a}, \beta$; 3) find $g_i^*$ synthetic samples belonging to area $\gZ_i$ that maximize  $Target_i$.

 In the first step, to prevent synthetic samples that are too far from the real data points, we restrict the search space of the generated points to a sphere with center $\vs_i$ and radius  $r_i = \max\limits_{\vz_j \in kNN(\vs_i)}\|\vz_j-\vs_i\|,$ where $kNN(\vs_i)$ denotes the set of $k$ nearest samples to $\vs_i$ in the real dataset $\mS$, $k$ is a hyperparameter and $\|. \|$ is $L_2$ norm.

 \textbf{Computational complexity.} We next discuss the complexity of OSYN to generate $g$ synthetic samples. Each iteration of Algorithm~\ref{alg-2} involves four main steps: (1) generating a synthetic dataset $\mG^t$ of size $N$ at cost $O(N C_G)$, where $C_G$ is the per-sample generation time; (2) assigning each sample to the nearest one of $K$ regions with cost $O(NKV)$, using Euclidean distance, where $V$ is the sample dimensionality; (3) computing the model loss on all synthetic points at cost $O(N C_h)$, with $C_h$ being the evaluation time per input; and (4) selecting top $g_i^*$ points from each region which costs $O(K g_i^* N)$, and hence $O(K g N)$ overall for all regions. Thus, the total time complexity for each iteration is $O\left(N C_G + N C_h + NKV + gNK \right)$. These operations are highly parallelizable. Furthermore, due to typically small $K$ and constant $C_G$, $C_h$, the algorithm can scale well to large synthetic datasets.


\section{Experiments}

 In this section, we assess the performance of our OSYN method under three key aspects: (i) generator quality, (ii) classifier robustness, and (iii) small-test sampling quality. For each scenario, we first describe the experimental setup, including baselines, datasets, and implementation details. While our experiments focus on classification problems, we also provide results for regression problems and additional ablation studies in Appendix~\ref{app-more-Ablation}.

 Results are summarized in the corresponding tables. Each method column (OSYN and baselines) reports $\text{mean}\pm\text{std}$ over at least five independent runs, followed by a gap in parentheses. $\text{Gap}_{\text{method}}=\text{Oracle Loss}-\text{Estimated Loss}_{\text{method}}$ is the signed deviation between the method's estimate ($\text{Estimated Loss}_{\text{method}}$) and the true loss ($\text{Oracle Loss}$). So positive values indicate valid lower bounds (underestimation) and negative values indicate overestimation. Within each row, \textbf{boldface} marks the \emph{tightest from-below} estimate; if no method falls below $\text{Oracle Loss}$, we bold the smallest absolute gap $\lvert\text{Gap}_{\text{method}}\rvert$. Asterisks (*) indicate baselines for which OSYN attains a significantly smaller signed gap 
than the baseline under a one-sided independent-samples Student's $t$-test ($\alpha=0.01$).
\subsection{The effect of the generator quality on the lower bound} \label{sil-ass-qua-gen-qua}
Assessing the quality of a generator $\mG$ in real-world applications is challenging without access to the true data distribution. To address this limitation, we create a synthetic dataset using a parameterized Gaussian mixture model, which allows for precise control over the underlying distribution. This controlled environment enables us to systematically manipulate generator deviation and directly observe its effects on the accuracy of our lower bound error estimates.

 \textbf{Dataset.} We simulate \(P_0 = (P_{\mathbf{x}}, P_y)\) for multi-class classification problem, where $
P_{\mathbf{x}}(\mathbf{x}) = \sum_{k=1}^5 \pi_k \mathcal{N}(\mathbf{x}\mid\mu_k,\Sigma_k)$ is a Gaussian mixture model (GMM) of five 2D Gaussians. Each Gaussian represents one class. Then, we sample from $P_0$ three discrete datasets $D_{train}, \mS, D_{oracle}$ such that $|D_{oracle}| >> |D_{train}| >> |\mS|$. The set $D_{train}$ is used to train the classifier $\vh$ and $D_{oralce}$ is used to compute the \textbf{oracle loss} $F(D_{oracle}, \vh)$, which describes the true loss $F(P_0, \vh)$, while $\mS$ represents the small test set in practice. Details on experiment settings are in Appendix \ref{simulated-data}.



 \textbf{Experimental details.} To illustrate the change in generator quality, we shift the distribution $P_0$ to build a generator. For simplicity, we only shift the means of the Gaussians in $P_{\vx}$, while keeping the labels and covariance matrices unchanged. The changed means of the GMM are:   $\mu_k^a = \mu_k + a\mu_0$, for each $k \in [5]$,   where $\mu_0 = [1, 0]$ and $a$ is the shift scale.
We measure the quality of the generator through the Kullback-Leibler (KL) divergence between $P_{\gG}^a= (P_{\vx}^a, P_y^a)$ and $P_0$. The larger the KL divergence, the worse the quality of the generator, and the more the distribution $P_\gG$  differs from $P_0$. For each shift scale $a$, we first use using Algorithm~\ref{alg-2} to find the optimized synthetic data from $P_{\gG}^a$ distribution, and then compute the lower bound $\text{LB}_a$  and  the gap $\textbf{Gap} = F(D_{oracle}, \vh)-\text{LB}_a$. The smaller the gap, the better the quality of the evaluation. The used loss $\ell$ is Zero-One loss. 
\begin{figure}[!t]  
\centering
\begin{minipage}[t]{0.48\textwidth}\vspace{0pt}
\centering
\captionof{table}{Lower Bound Quality optimized by OSYN vs.\ Generator Quality}
\label{tab-1}
\small
\resizebox{\linewidth}{!}{%
\begin{tabular}{c c c c}
\toprule
$a$ & \textbf{OSYN} & \textbf{Gap} & $KL(P_\gG^{a},P_{0})$\\
\midrule
$0.00$   & $0.056${\tiny{$\pm$0.000}} & $0.015${\tiny{$\pm$0.000}} & $0.000$ \\
$-0.25$  & $0.057${\tiny{$\pm$0.000}} & $0.014${\tiny{$\pm$0.000}} & $0.011$ \\
$-0.50$  & $0.056${\tiny{$\pm$0.001}} & $0.016${\tiny{$\pm$0.001}} & $0.042$ \\
$-0.75$  & $0.054${\tiny{$\pm$0.001}} & $0.017${\tiny{$\pm$0.001}} & $0.092$ \\
$-1.00$  & $0.052${\tiny{$\pm$0.002}} & $0.019${\tiny{$\pm$0.002}} & $0.160$ \\
$-1.125$ & $0.050${\tiny{$\pm$0.000}} & $0.021${\tiny{$\pm$0.000}} & $0.200$ \\
$-1.25$  & $0.047${\tiny{$\pm$0.001}} & $0.024${\tiny{$\pm$0.001}} & $0.244$ \\
$-1.50$  & $0.045${\tiny{$\pm$0.001}} & $0.027${\tiny{$\pm$0.001}} & $0.340$ \\
$-1.75$  & $0.040${\tiny{$\pm$0.000}} & $0.031${\tiny{$\pm$0.000}} & $0.446$ \\
$-2.00$  & $0.035${\tiny{$\pm$0.001}} & $0.037${\tiny{$\pm$0.001}} & $0.558$ \\
\bottomrule
\end{tabular}}
\end{minipage}
\begin{minipage}[t]{0.48\textwidth}\vspace{0pt}
\centering
\includegraphics[width=\linewidth, height=5cm, keepaspectratio]{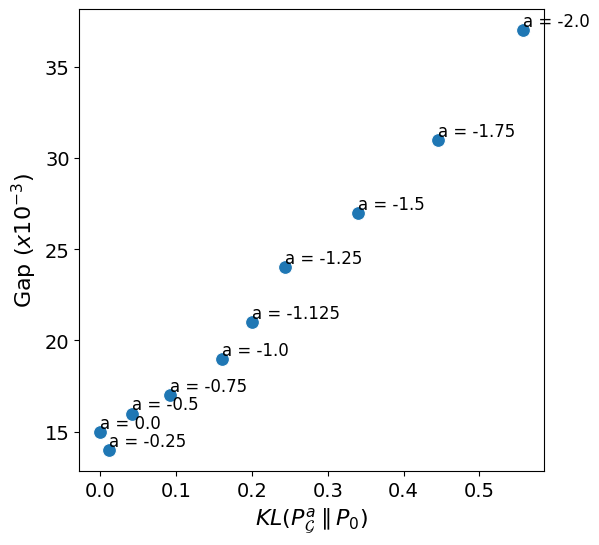}
\captionof{figure}{KL Divergence vs. Gap}
\label{fig-1}
\end{minipage}
\end{figure}

\textbf{Results. } Table \ref{tab-1} shows the optimized results of 10 different scale values for the Decision Tree model. As $|a|$ increases, the KL divergence also increases, and the gap between the lower bound and the true error becomes larger. There is a strongly linear relationship between optimized gaps and KL values (Figure~\ref{fig-1}) with their Pearson correlation is $0.994$. \textit{This indicates that as the quality of the generator deteriorates, the generated distribution deviates more from the original data distribution, and the quality of the evaluation also declines.} Conversely, when the generator quality is high, we can obtain a reliable estimate of the true error. This experiment once again highlights the importance of generator quality, as discussed in the theoretical analysis above.


\subsection{Evaluation on small and biased test set} \label{sil-ass-qua-model-qua}
 
In this part, we investigate whether, for a fixed generator $\mG$, the quality of classifier $\vh$ correlates with the tightness of our lower-bound error estimate, and we characterize how this relationship manifests. To do so, we track the optimized lower bound across a range of classifiers in two experimental scenarios: the controlled GMM simulations described earlier, and a collection of real-world tabular datasets.

\subsubsection{Experimental settings} \label{real-model-quality-eval-qual}
\textbf{Baselines.} We benchmark  OSYN  against two widely-used baselines for estimating the true error:
\begin{itemize}
    \item \textit{Bootstrap Loss:} repeatedly resample the small test set $\mS$ with replacement, evaluate the model at confidence level $1-\delta$, and take the $\delta$ percentile of observed losses.
    \item \textit{Syn-wo-Opt Loss:} draw a dataset directly from the generator and measure the model’s error on this dataset, without any modification or optimization.
\end{itemize}

{\setlength{\extrarowheight}{2pt}
\begin{table}[t]
\caption{Results on simulation data. 
}
\label{tab-2}
\begin {tabular}{l c l l c}
\hline
\textbf{Model} & \textbf{OSYN}  & \textbf{Bootstrap Loss} & \textbf{Syn-wo-Opt Loss} & \textbf{Oracle Loss}\\ 
\hline 
kNN& $0.043${\tiny{$\pm$0.000}} $(+0.011)$  
& $0.049 (+0.005)$ & $\mathbf{0.054}${\tiny{$\pm$0.000}} $(+0.000)$ & $0.054$ \\
SVM & $0.038${\tiny{$\pm$0.002}} $(+0.007)$  
& $\mathbf{0.044} (+0.002)$ & $0.046${\tiny{$\pm$0.001}} $(-0.001)$ & $0.045$ \\
DT& 
$0.061${\tiny{$\pm$0.002}} $(+0.011)$  
& $0.059 (+0.012)^{*}$ & $\mathbf{0.070}${\tiny{$\pm$0.000}} $(+0.002)$ & $0.071$\\
ANN& $\mathbf{0.047}${\tiny{$\pm$0.001}}  $(+0.001)$ 
& $0.044 (+0.003)^{*}$ & $0.048$ {\tiny{$\pm$0.001}} $(-0.001)$ &$0.047$\\
RF& $0.044${\tiny{$\pm$0.002}} $(+0.005)$ 
&$0.050 (-0.002)$
 & $\mathbf{0.049}${\tiny{$\pm$0.001}} $ (+0.000)$ & $0.049$\\
LR & $\mathbf{0.044} ${\tiny{$\pm$0.002}} $ (+0.002)$ 
& 
$0.049 (-0.004)^{*}$ 
& $0.047 ${\tiny{$\pm$0.000}} $(-0.002)$ &$0.046$\\
NB & $\mathbf{0.043}${\tiny{$\pm$0.001}} $(+0.004)$
&$0.076 (-0.029)^{*}$ & $0.049${\tiny{$\pm$0.000}} $(-0.002)$ & $0.047$\\
QDA & $\mathbf{0.042}${\tiny{$\pm$0.001}} $(+0.001)$ 
& $0.044 (-0.001)$ & $0.045${\tiny{$\pm$0.001}} $(-0.002)^{*}$ & $0.043$\\
\hline 
\end{tabular}
\end{table}
}

\textbf{Datasets.} In addition to our controlled simulations, we assess  OSYN  on three standard tabular datasets for binary classification. The CREDIT \cite{yeh2009comparisons} dataset labels loan applicants as credit-worthy or not, the BANK \cite{moro2014data} dataset predicts term-deposit subscription following a campaign, and the ADULT \cite{adult_2} census-income dataset classifies individuals by whether their annual income exceeds $\$50,000$. For each real-world dataset, we perform a stratified random split into a training set $D_{train}$ ($30\%$) and an oracle set $D_{oracle}$ ($70\%$), ensuring class proportions are preserved. From $D_{oracle}$, we then draw a small test sample $\mS$ of 500 instances, comprising 300 majority-class and 200 minority-class examples. More information about the datasets can be found at Appendix~\ref{real-tabular-data}.

\textbf{Experimental details. }In the simulated-data experiments, we fit eight standard classifiers from Scikit-Learn: k-Nearest Neighbors (kNN), Linear SVM, Decision Tree (DT), Random Forest (RF), Neural Net (ANN), Logistic Regression (LR), Naive Bayes (NB), and Quadratic Discriminant Analysis (QDA) — using default library hyperparameters on the training set $D_{train}$. For the real-tabular datasets, we substitute NB and QDA with Gradient Boosting (GB) and Linear Discriminant Analysis (LDA) under the same training regime. Synthetic samples are produced by a CTGAN \cite{xu2019modeling} model fine-tuned on the oracle set $D_{oracle}$. To approximate the generator distribution $P_\gG$, we partition the feature space into $|\mS|$ Voronoi cells centered at each test point in $\mS$, employing FAISS \cite{douze2024faiss} for efficient clustering. We then generate $N=1,000,000$ points from CTGAN, estimate $P_\gG$  via region‐wise frequencies, and select the scaling parameter $b$ by grid search per dataset. Our proposed OSYN method runs for $T=15$ iterations, each producing $50,000$ labeled samples, and applies a fixed combined confidence threshold $\delta = \delta_1 + \delta_2$. The Bootstrap Loss baseline uses 2000 resamples with replacement, takes the $\delta$ percentile of loss. The Syn-wo-Opt calculates the average loss value of $g^*$ random synthetic points without optimization. All methods use the zero–one loss for error calculation, and each configuration is repeated ten times (for tabular data) or three times (for simulated data) to report mean $\pm$ standard deviation.

\subsubsection{Main results} \label{resutl}

\textbf{Simulated data.} Table \ref{tab-2} shows the optimized results of 8 commonly used classifiers. Our method provides lower bounds that are quite close to the true loss for most models. In contrast, the bootstrap-based estimation yields unstable results; some models have bootstrap lower bound values higher than the true loss, while others are lower. A similar pattern is observed with a method using synthetic loss without optimization. This experimental result clearly demonstrates the effectiveness of searching for optimal synthetic data based on theoretical guarantees. If we only use the error of the generated samples to estimate the true error, the resulting value may fluctuate significantly, either being lower than or higher than the oracle loss, and the degree of deviation depends on the quality of the generator. Similarly, if the bootstrap method is used to estimate the error, the resulting value heavily depends on the quality of the small test set $\mS$. If $\mS$ is biased or small in size, the obtained evaluation will not be reliable.









{\setlength{\extrarowheight}{2pt}
\begin{table}[t]
\centering
\small
\caption{Results on real-life datasets (CREDIT, BANK, ADULT). 
}
\label{dif-models-combined-datasets}
\setlength{\tabcolsep}{4pt}   
\renewcommand{\arraystretch}{0.8} 
\begin{tabular}{
    l c l l c}                                
\toprule
\multicolumn{1}{c}{\textbf{Model}} & 
\multicolumn{1}{c}{\textbf{OSYN}} & 
\multicolumn{1}{c}{\textbf{Bootstrap Loss}} & 
\multicolumn{1}{c}{\textbf{Syn-wo-Opt Loss}} & 
\multicolumn{1}{c}{\textbf{Oracle Loss}} \\
\midrule
\midrule
\multicolumn{5}{l }{CREDIT} \\
\midrule
  kNN
  &$\mathbf{0.239}${\tiny$\pm0.001$} $(+0.009)$ 
  & $0.326 (-0.078)^{*}$ 
  & $0.283${\tiny$\pm0.000$} $(-0.034)^{*}$ 
  & $0.248$ \\
  SVM  
  & $\mathbf{0.220}${\tiny$\pm0.000$} $(+0.005)$ 
  & $0.342 (-0.118)^{*}$ 
  & $0.298${\tiny$\pm0.000$} $(-0.073)^{*}$ 
  & $0.224$ \\
  DT   
  &$\mathbf{0.230}${\tiny$\pm0.000$} $(+0.042)$ 
  & 0.292 ($-0.020)$ 
  & $0.356${\tiny$\pm0.001$} $(-0.083)^{*}$ 
  & $0.272$ \\
  ANN  
  &$\mathbf{0.208}${\tiny$\pm0.000$} $(+0.001)$ 
  & 0.342 ($-0.124)^{*}$ & $0.289${\tiny$\pm0.000$} $(-0.071)^{*}$ & $0.219$ \\
  RF   
  & $\mathbf{0.196}${\tiny$\pm0.000$} $(+0.007)$ 
  & 0.322 ($-0.119)^{*}$ & $0.240${\tiny$\pm0.001$} $(-0.036)^{*}$ & $0.203$ \\
  LR   
  & $\mathbf{0.211}${\tiny$\pm0.000$} $(+0.003)$ 
  & 0.344 ($-0.130)^{*}$ & $0.255${\tiny$\pm0.000$} $(-0.041)^{*}$ & $0.214$ \\
  GB   
  &$\mathbf{0.130}${\tiny$\pm0.000$} $(+0.045)$ 
  & 0.234 ($-0.060)^{*}$ & $0.209${\tiny$\pm0.001$} $(-0.035)$ & $0.174$ \\
  LDA  
  & $\mathbf{0.139}${\tiny$\pm0.000$} $(+0.045)$ 
  & 0.262 ($-0.078)^{*}$ & $0.206${\tiny$\pm0.000$} $(-0.023)^{*}$ & $0.184$ \\
\midrule
  \multicolumn{5}{l }{BANK} \\
\midrule
  kNN  
  & $\mathbf{0.096}${\tiny$\pm0.002$} $(+0.006)$ 
  & 0.253 ($-0.151)^{*}$ 
  & $0.215${\tiny$\pm0.011$} $(-0.113)^{*}$ 
  & $0.102$ \\
  SVM  
  & $\mathbf{0.105}${\tiny$\pm0.003$} $(+0.002)$ 
  & 0.272 ($-0.165)^{*}$ 
  & $0.189${\tiny$\pm0.000$} $(-0.082)^{*}$ 
  & $0.108$ \\
  DT
  & $\mathbf{0.089}${\tiny$\pm0.003$} $(+0.006)$ 
  & 0.202 ($-0.108)^{*}$ 
  & $0.243${\tiny$\pm0.029$} $(-0.149)^{*}$ 
  & $0.094$ \\
  ANN  
  & $\mathbf{0.078}${\tiny$\pm0.002$} $(+0.026)$ 
  & 0.240 ($-0.135)^{*}$ 
  & $0.185${\tiny$\pm0.001$} $(-0.080)^{*}$ 
  & $0.105$ \\
  RF   
  & $\mathbf{0.115}${\tiny$\pm0.003$} $(+0.003)$ 
  & 0.350 ($-0.233)^{*}$ 
  & $0.214${\tiny$\pm0.001$} $(-0.097)^{*}$ 
  & $0.117$ \\
  LR   
  & $\mathbf{0.101}${\tiny$\pm0.002$} $(+0.006)$ 
  & 0.247 ($-0.141)^{*}$
  & $0.186${\tiny$\pm0.001$} $(-0.079)^{*}$ 
  & $0.107$ \\
  GB  
  & $\mathbf{0.085}${\tiny$\pm0.001$} $(+0.011)$ 
  & 0.208 ($-0.111)^{*}$ 
  & $0.184${\tiny$\pm0.001$} $(-0.087)^{*}$ 
  & $0.097$ \\
  LDA  
  & $\mathbf{0.091}${\tiny$\pm0.001$} $(+0.017)$ 
  & 0.229 ($-0.120)^{*}$ 
  & $0.222${\tiny$\pm0.020$} $(-0.113)^{*}$ 
  & $0.109$ \\
\midrule
\multicolumn{5}{l }{ADULT} \\
\midrule
  kNN  
  & $0.214${\tiny$\pm0.001$} $(+0.006)$ 
  & 0.300 ($-0.080)^{*}$ & $\mathbf{0.216}${\tiny$\pm0.000$} $(+0.004)$ & $0.220$ \\
  SVM  
  & $\mathbf{0.200}${\tiny$\pm0.001$} $(+0.017)$ 
  & 0.316 ($-0.099)^{*}$ & $0.190${\tiny$\pm0.000$} $(+0.026)^{*}$ & $0.217$ \\
    DT   
    & $\mathbf{0.177}${\tiny$\pm0.001$} $(+0.018)$ 
    & 0.197 ($-0.003)$ & $0.207${\tiny$\pm0.001$} $(-0.013)$ & $0.194$ \\
  ANN  & $\mathbf{0.198}${\tiny$\pm0.001$} $(+0.011)$ 
  & 0.271 ($-0.062)^{*}$ & $0.193${\tiny$\pm0.002$} $(+0.016)^{*}$ & $0.209$ \\
  RF   & $\mathbf{0.185}${\tiny$\pm0.000$} $(+0.012)$ 
  & 0.284 ($-0.087)^{*}$ 
  & $0.180${\tiny$\pm0.000$} $(+0.018)^{*}$ & $0.197$ \\
  LR   & $0.192${\tiny$\pm0.001$} $(+0.017)$ 
  & 0.277 ($-0.068)^{*}$ & $\mathbf{0.210}${\tiny$\pm0.001$} $(+0.000)$ & $0.210$ \\
  GB   & $\mathbf{0.109}${\tiny$\pm0.001$} $(+0.026)$ 
  & 0.150 ($-0.016)$ & $0.142${\tiny$\pm0.000$} $(-0.008)$ & $0.134$ \\
  LDA  & $\mathbf{0.172}${\tiny$\pm0.001$} $(+0.011)$ 
  & 0.224 ($-0.041)^{*}$ & $0.157${\tiny$\pm0.001$} $(+0.027)^{*}$ & $0.184$ \\
\bottomrule
\end{tabular}
\end{table}
}

\textbf{Real-world datasets.} Table \ref{dif-models-combined-datasets} shows the optimized lower bounds of different models compared to two baselines. The obtained lower bounds estimated by OSYN are often close to the true loss for most models, similar to the results on the simulation data above. On the other hand, most of the lower bounds produced by the bootstrap method are inaccurate, with values greater than the true loss. This reflects a limitation of using the bootstrap method to evaluate quality, as the set $\mS$ in this case is highly biased with $F(\mS, \vh) \gg F(D_{\text{oracle}}, \vh)$, which leads to poor evaluation results—the lower bound ends up higher than the true errors.

 We observe that sometimes using the non-optimal synthetic dataset to estimate the true error can outperform OSYN and  Bootstrap, for the kNN and Logistic Regression models for the ADULT dataset. However, the difference is not significant. In all the remaining models and datasets, the non-optimal synthetic loss produces inaccurate results. Meanwhile, the estimates by OSYN, despite the imperfections of the generator, still provide reliable evaluations thanks to the guides from our theoretical analyses.

\subsection{Evaluation on small but balanced test set} \label{data-quality}


We next assess the impact of test‐set characteristics by varying two factors: (i) the number of samples in $\mS$ and (ii) the class‐label balance.  

\textbf{Experimental details.} We utilize the same experimental framework applied in our previous tabular data tests with the ADULT dataset, ensuring that we train and evaluate using identical baselines and CTGAN settings. To isolate the impact of class-label balance, we create a fully balanced test set, $\mS$, comprising 500 instances—250 examples per class—instead of the original 3:2 majority-biased split. Additionally, to investigate the effect of test set size, we vary the size of $|\mS|$ from 300 to 700 samples (maintaining the original 3:2 label ratio) and evaluate the  GB classifier in this regime. 

{\setlength{\extrarowheight}{2pt}
\begin{table}
\caption{Results on ADULT dataset, where the test set is \textit{balanced}. 
}
\label{balance-s-adult}
\begin {tabular}{l  l l l c}
\hline
\multicolumn{1}{c}{\textbf{Model}} & 
\multicolumn{1}{c}{\textbf{OSYN}} & 
\multicolumn{1}{c}{\textbf{Bootstrap Loss}} & 
\multicolumn{1}{c}{\textbf{Syn-wo-Opt Loss}} & 
\multicolumn{1}{c}{\textbf{Oracle Loss}} \\
\hline 
kNN& $\mathbf{0.216}${\tiny{$\pm$0.000}} $(+0.003)$ & $0.349 (-0.131)^{*}$ & $0.194${\tiny{$\pm$0.000}} $(+0.024)^{*}$& $0.218$ \\

SVM & $\mathbf{0.218}${\tiny{$\pm$0.000}} $(+0.003)$ & $0.410(-0.189)^{*}$ & $0.173${\tiny{$\pm$0.000}} $(+0.048)^{*}$& $0.221$ \\

DT& $\mathbf{0.180}${\tiny{$\pm$0.001}} $(+0.014)$ & $0.219 (-0.025)^{*}$ & $0.211${\tiny{$\pm$0.001}} $(-0.017)^{*}$& $0.194$ \\

ANN& $\mathbf{0.196}${\tiny{$\pm$0.000}} $(+0.014)$ & $0.322 (-0.112)^{*}$ & $0.191${\tiny{$\pm$0.001}} $(+0.019)^{*}$& $0.210$ \\

RF& $\mathbf{0.192}${\tiny{$\pm$0.000}} $(+0.004)$ & $0.360 (-0.165)^{*}$ & $0.152${\tiny{$\pm$0.000}} $(+0.044)^{*}$& $0.195$ \\

LR & $\mathbf{0.197}${\tiny{$\pm$0.000}} $(+0.011)$ & $0.327 (-0.118)^{*}$ & $0.195${\tiny{$\pm$0.000}} $(+0.014)^{*}$& $0.208$ \\

GB& $\mathbf{0.128}${\tiny{$\pm$0.000}} $(+0.006)$ & $0.182 (-0.048)^{*}$ & $0.138${\tiny{$\pm$0.000}} $(-0.005)$& $0.134$ \\
LDA & $\mathbf{0.158}${\tiny{$\pm$0.000}} $(+0.025)$ & $0.266 (-0.084)^{*}$ & $0.148${\tiny{$\pm$0.000}} $(+0.034)^{*}$& $0.182$ \\
\hline 
\end{tabular}
\end{table}
}

\textbf{Results.} Table \ref{balance-s-adult} shows the results of our method compared to 2 baselines for the ADULT dataset with $\mS$ of size 500. Our method outperforms the two baselines, achieving the closest optimal lower bounds to the true loss across all 8 classification models. The bootstrap method continues to yield values higher than the oracle loss, as it heavily relies on the value of $F(\mS, \vh)$; if the error on set $\mS$ deviates significantly from the oracle loss, bootstrap results become inaccurate. Meanwhile, the method that only estimates the true loss based on the generated set without optimization also fails to produce reliable results, as the obtained values are unstable, sometimes higher, sometimes lower than the oracle loss. Additionally, our lower bounds tend to be closer to the oracle loss when the set $\mS$ is balanced compared to when $\mS$ is imbalanced. The average gap across the 8 models in the imbalanced case is $0.0152$, whereas in the balanced case, it is $0.0083$.

\begin{figure}[t!]
\centering
\includegraphics[scale= 0.4]{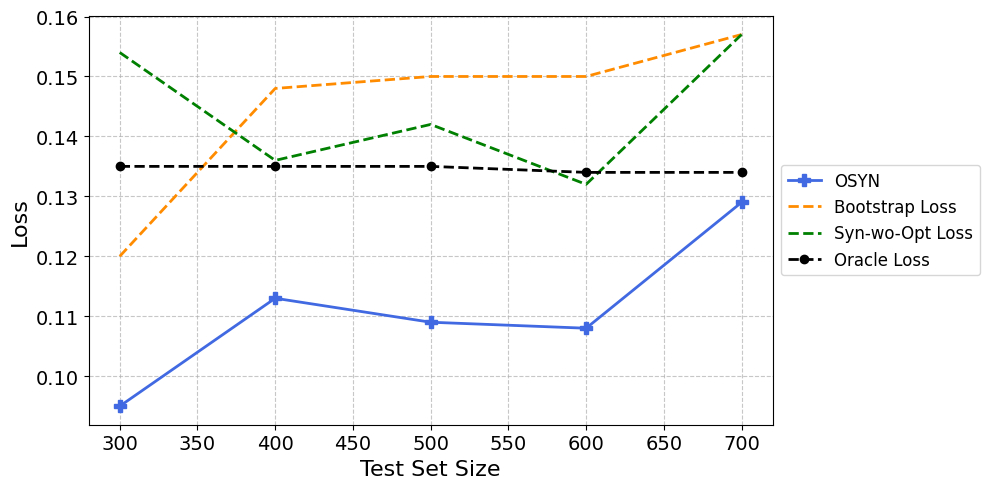} 
\caption{Results when changing the test size.}
\label{fig-3}
\end{figure}
 Figure~\ref{fig-3} depicts the results of the gradient boost model for the ADULT data set with 5 different sizes of $\mS$. As the size of the set $\mS$ increases, the lower bound gradually approaches the oracle loss, and the evaluation quality increases. Meanwhile, the results based on bootstrap and synthetic loss remain unstable, with significant fluctuations even when the size of $\mS$ is large.

\subsection{Ablation studies}
This subsection investigates how the choice of some factors influences the accuracy of of OSYN. 

\subsubsection{Effectiveness of different generators} \label{abla_generators}

We first analyze two key factors: (i) the type of generative model used, and (ii) the dataset used to train the generator. Our experiments involve four widely-used table data generators: CTGAN and TVAE \cite{xu2019modeling}, Copula GAN and Gaussian Copula (GC) \cite{patki2016synthetic}, all implemented by the open-source SDV Python package \cite{patki2016synthetic}. We evaluate their performance using the BANK dataset, under the same configuration and data split as in Section \ref{sil-ass-qua-model-qua}. To ensure robust conclusions, each experimental setting is repeated five times. We report results for three representative classifiers from Section~\ref{sil-ass-qua-model-qua}. Details of the effect of the \emph{generator's training dataset} on lower-bound quality are provided in Appendix~\ref{abla_gen_2}.

{\setlength{\extrarowheight}{2pt}
\begin{table}[!t]
\centering
\fontsize{7pt}{8.5pt}\selectfont
\caption{Evaluation results when using different generators. 
}
\label{table: th1_generator}
\setlength{\tabcolsep}{3pt}   
\begin{tabular}{
    >{\raggedright\arraybackslash}p{0.5cm}  
    l                                       
    l l l c}         
\toprule
\multicolumn{1}{c}{\textbf{Model}} & 
\multicolumn{1}{c}{\textbf{Generator}} & 
\multicolumn{1}{c}{\textbf{OSYN}} & 
\multicolumn{1}{c}{\textbf{Bootstrap Loss}} & 
\multicolumn{1}{c}{\textbf{Syn-wo-Opt Loss}} & 
\multicolumn{1}{c}{\textbf{Oracle Loss}} \\
\midrule
\multirow{4}{*}{ANN}
  & TVAE & $\mathbf{0.063}${\tiny$\pm0.000$} $(+0.041)$ & 0.240 $(-0.135)^{*}$ & $0.123${\tiny$\pm0.000$} $(-0.018)$ & $0.105$\\
  & GC & $\mathbf{0.092}${\tiny$\pm0.000$} $(+0.013)$ & 0.240 $(-0.135)^{*}$ & $0.145${\tiny$\pm0.000$} $(-0.040)^{*}$ & $0.105$\\
  & Copula & $\mathbf{0.094}${\tiny$\pm0.000$} $(+0.011)$ & 0.240 $(-0.135)^{*}$ & $0.192${\tiny$\pm0.000$} $(-0.087)^{*}$ & $0.105$\\
  & CTGAN & $\mathbf{0.095}${\tiny$\pm0.000$} $(+0.010)$ & 0.240 $(-0.135)^{*}$ & $0.141${\tiny$\pm0.000$} $(-0.036)^{*}$ & $0.105$\\
\midrule
\multirow{4}{*}{GB}
  & TVAE & $\mathbf{0.058}${\tiny$\pm0.000$} $(+0.039)$ & 0.208 $(-0.111)^{*}$ & $0.125${\tiny$\pm0.000$} $(-0.028)$ & $0.097$\\
  & GC & $\mathbf{0.080}${\tiny$\pm0.000$} $(+0.016)$ & 0.208 ($-0.111)^{*}$ & $0.181${\tiny$\pm0.001$} $(-0.084)^{*}$ & $0.097$\\
  & Copula & $\mathbf{0.084}${\tiny$\pm0.000$} $(+0.012)$ & 0.208 ($-0.111)^{*}$ & $0.195${\tiny$\pm0.000$} $(-0.098)^{*}$ & $0.097$\\
  & CTGAN & $\mathbf{0.084}${\tiny$\pm0.000$} $(+0.013)$ & 0.208 ($-0.111)^{*}$ & $0.143${\tiny$\pm0.000$} $(-0.047)^{*}$ & $0.097$\\
\midrule
\multirow{4}{*}{LDA}
  & TVAE & $\mathbf{0.070}${\tiny$\pm0.000$} $(+0.038)$ & 0.228 ($-0.119)^{*}$ & $0.124${\tiny$\pm0.000$} $(-0.015)$ & $0.109$\\
  & GC & $\mathbf{0.085}${\tiny$\pm0.000$} $(+0.024)$ & 0.228 ($-0.119)^{*}$ & $0.158${\tiny$\pm0.001$} $(-0.049)^{*}$ & $0.109$\\
  & Copula & $\mathbf{0.089}${\tiny$\pm0.000$} $(+0.020)$ & 0.228 ($-0.119)^{*}$ & $0.200${\tiny$\pm0.000$} $(-0.092)^{*}$ & $0.109$\\
  & CTGAN & $\mathbf{0.090}${\tiny$\pm0.000$} $(+0.019)$ & 0.228 ($-0.119)^{*}$ & $0.121${\tiny$\pm0.000$} $(-0.012)$ & $0.109$\\
\bottomrule
\end{tabular}
\end{table}
}

\textit{Effect of different generative models:} We first assess the impact of generator choice by training all four generative models on the oracle dataset $D_{\mathrm{oracle}}$ for 200 epochs, using a batch size of 500 and default hyperparameters. The results are summarized in Table~\ref{table: th1_generator}. 

As shown in Table \ref{table: th1_generator}, OSYN consistently achieves the smallest gap to the true loss, outperforming both Bootstrap and Syn-wo-Opt across all classifiers and generators. Interestingly, while TVAE yields the smallest Syn-wo-Opt error, suggesting a strong fit to the real data, it performs worse in OSYN-based estimation, with a noticeably larger deviation from the true loss. This likely reflects a mismatch between TVAE’s current hyperparameter setting. In contrast, both CTGAN and Copula GAN offer more stable lower-bound estimates under OSYN, with performance that is comparable and superior to GC and TVAE. These results suggest that although generator quality affects estimation to some extent, OSYN remains robust across a range of generative models and does not overly rely on the fidelity of the synthetic data distribution.

\subsubsection{Effectiveness of the number of partitions} \label{abla_partition}
We next evaluate how the size $K$ of the partition affects the performance of the OSYN method. All experiments are conducted on the BANK dataset and repeated five times. The data split, configuration, and baseline methods, and CTGAN as the generator trained on the oracle dataset \(D_{\mathrm{oracle}}\) follow the setup in Section \ref{sil-ass-qua-model-qua}. We vary \(K\) across \{50, 100, 200, 300, 500\}, while keeping the adjustment parameter \(b\) fixed as described in Section~\ref{osyn_method}.

The results are shown in Table~\ref{table:th3_partition}. We observe that OSYN consistently produces the tightest lower bounds compared to both Bootstrap and Syn-wo-Opt, regardless of the value of $K$. In the case of the ANN classifier, the good performance is achieved when $K = 50$, with a lower bound of $0.082 \pm 0.001$ and a gap of only $+0.022$ from the true loss—substantially outperforming the two baselines. A similar trend is observed for LDA, where $K = 50$ yields the best bound ($0.104 \pm 0.005$), closest to the oracle loss of $0.109$.

{\setlength{\extrarowheight}{2pt}
\begin{table}[!t]
\centering
\small
\caption{Performance when changing the size $K$ of the partition.}
\label{table:th3_partition}
\setlength{\tabcolsep}{3pt}   
\begin{tabular}{
    >{\raggedright\arraybackslash}p{1cm}  
    l                                       
    l l l c}         
\toprule
\multicolumn{1}{c}{\textbf{Model}} & 
\multicolumn{1}{c}{\textbf{K}} & 
\multicolumn{1}{c}{\textbf{OSYN}} & 
\multicolumn{1}{c}{\textbf{Bootstrap Loss}} & 
\multicolumn{1}{c}{\textbf{Syn-wo-Opt Loss}} & 
\multicolumn{1}{c}{\textbf{Oracle Loss}} \\
\midrule
\multirow{4}{*}{ANN}
  & 50 & $\mathbf{0.082}${\tiny$\pm0.001$} $(+0.022)$ & 0.266 $(-0.161)^{*}$ & $0.185${\tiny$\pm0.001$} $(-0.087)^{*}$ & $0.105$\\
  & 100 & $\mathbf{0.075}${\tiny$\pm0.004$} $(+0.029)$ & 0.244 $(-0.139)^{*}$ & $0.184${\tiny$\pm0.001$} $(-0.080)^{*}$ & $0.105$\\
  & 200 & $\mathbf{0.061}${\tiny$\pm0.002$} $(+0.044)$ & 0.240 $(-0.135)^{*}$ & $0.185${\tiny$\pm0.001$} $(-0.087)^{*}$ & $0.105$\\
  & 300 & $\mathbf{0.075}${\tiny$\pm0.003$} $(+0.030)$ & 0.244 $(-0.139)^{*}$ & $0.185${\tiny$\pm0.001$} $(-0.080)^{*}$ & $0.105$\\
  & 500 & $\mathbf{0.097}${\tiny$\pm0.000$} $(+0.008)$ & 0.240 $(-0.135)^{*}$ & $0.184${\tiny$\pm0.000$} $(-0.080)^{*}$ & $0.105$\\
\midrule
\multirow{4}{*}{GB}
  & 50 & $\mathbf{0.079}${\tiny$\pm0.000$} $(+0.018)$ & 0.236 $(-0.139)^{*}$ & $0.185${\tiny$\pm0.000$} $(-0.088)^{*}$ & $0.097$\\
  & 100 & $\mathbf{0.090}${\tiny$\pm0.011$} $(+0.007)$ & 0.226 $(-0.129)^{*}$ & $0.184${\tiny$\pm0.001$} $(-0.087)^{*}$ & $0.097$\\
  & 200 & $\mathbf{0.094}${\tiny$\pm0.005$} $(+0.003)$ & 0.208 $(-0.111)^{*}$ & $0.184${\tiny$\pm0.001$} $(-0.087)^{*}$ & $0.097$\\
  & 300 & $\mathbf{0.081}${\tiny$\pm0.003$} $(+0.016)$ & 0.212 $(-0.115)^{*}$ & $0.184${\tiny$\pm0.000$} $(-0.088)^{*}$ & $0.097$\\
  & 500 & $\mathbf{0.092}${\tiny$\pm0.000$} $(+0.005)$ & 0.208 $(-0.111)^{*}$ & $0.184${\tiny$\pm0.001$} $(-0.087)^{*}$ & $0.097$\\
\midrule
\multirow{4}{*}{LDA}
  & 50 & $\mathbf{0.104}${\tiny$\pm0.001$} $(+0.005)$ & 0.258 $(-0.149)^{*}$ & $0.193${\tiny$\pm0.000$} $(-0.084)^{*}$ & $0.109$\\
  & 100 & $\mathbf{0.087}${\tiny$\pm0.007$} $(+0.021)$ & 0.228 $(-0.119)^{*}$ & $0.193${\tiny$\pm0.000$} $(-0.084)^{*}$ & $0.109$\\
  & 200 & $\mathbf{0.094}${\tiny$\pm0.009$} $(+0.015)$ & 0.230 $(-0.121)^{*}$ & $0.192${\tiny$\pm0.001$} $(-0.084)^{*}$ & $0.109$\\
  & 300 & $\mathbf{0.094}${\tiny$\pm0.000$} $(+0.015)$ & 0.234 $(-0.125)^{*}$ & $0.193${\tiny$\pm0.001$} $(-0.084)^{*}$ & $0.109$\\
  & 500 & $\mathbf{0.104}${\tiny$\pm0.000$} $(+0.005)$ & 0.228 $(-0.119)^{*}$ & $0.209${\tiny$\pm0.024$} $(-0.100)^{*}$ & $0.109$\\
\bottomrule
\end{tabular}
\end{table}
}

 For the GB classifier, while the overall bounds remain strong, the optimal $K$ is less clear. The bound tightness slightly improves as $K$ increases from 50 to 200, but the differences are marginal. Interestingly, for mid-range values of $K$ (e.g., 100–300), we observe a slight increase in variance, particularly in the ANN and GB models, suggesting that some instability may arise due to interactions between the partition size and the fixed optimization parameter $b$. At the extreme end, $K = 500$ (i.e., one partition per sample) generally leads to stable and competitive bounds, especially in GB, where the standard deviation is among the lowest. Overall, these results highlight the robustness of OSYN to the choice of $K$. While smaller $K$ values, such as 50, can yield marginally tighter estimates, larger $K$ values offer improved stability, suggesting that OSYN can adapt to a wide range of partitioning schemes without significant degradation in performance.

\subsubsection{Low-tuning heuristics for parameters} \label{abla_low_tune}
Lastly, we consider the regimes where good parameter settings often appear. Such regimes are important to enable an easy choice in practice.
The main parameters include: the partition size \(K\), the adjustment parameter \(b\), and the confidence parameter \(\delta_2\). We perform a grid search for \(K\in\{50,100,200,300,400,500\}\), \(b\in[0,7]\) (with a step of 0.1), and \(\delta_2 \in [10^{-4}, \delta-10^{-4}]\) (with a step of  \(10^{-4}\)).

\begin{figure}[!t]
  \centering
  \begin{subfigure}[t]{0.48\linewidth}
    \centering
    \includegraphics[width=\linewidth]{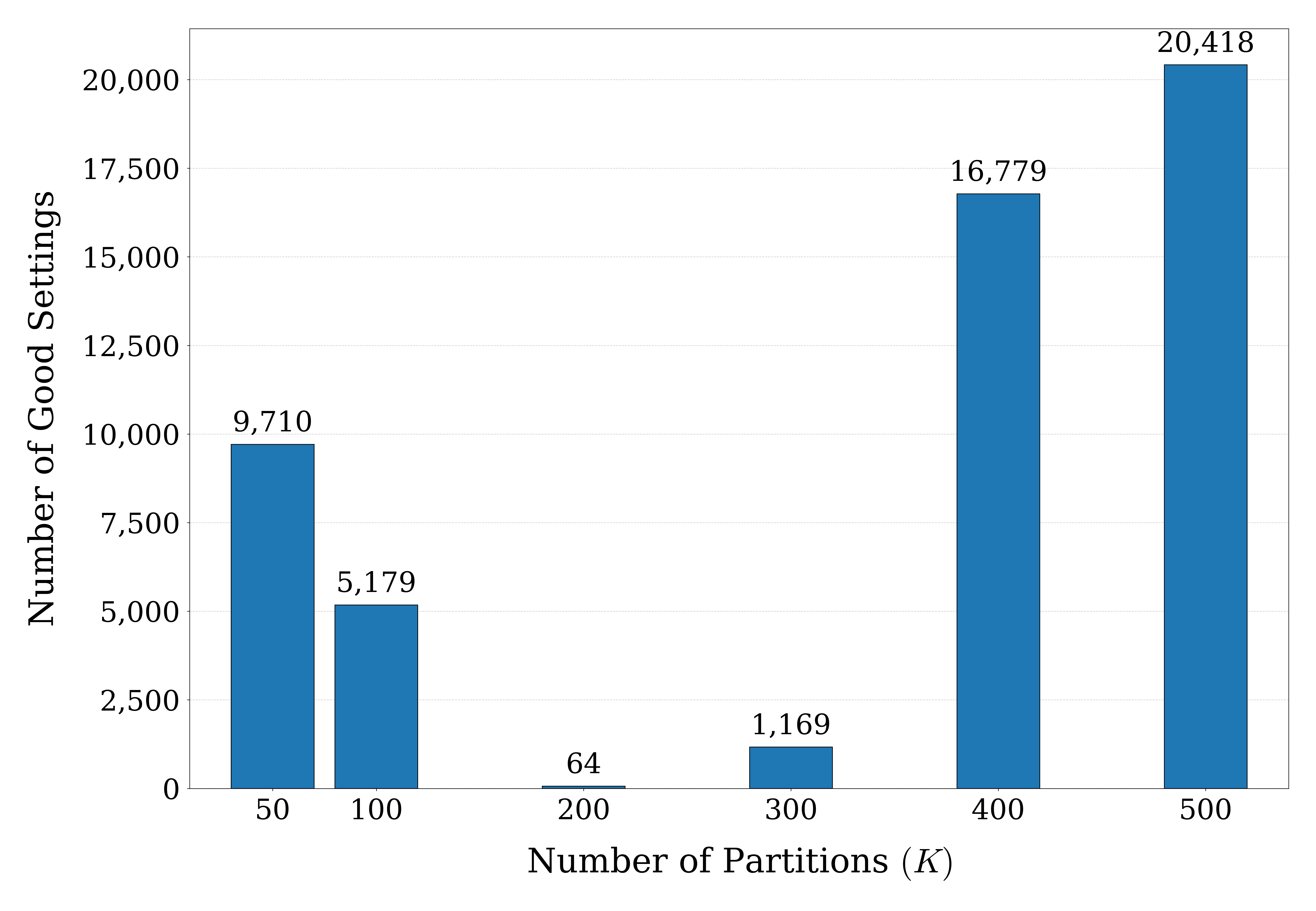}
    \caption{\textbf{Number of \emph{good} settings vs. partition size \(K\)}. The monotone increase supports using large \(K\) in practice.} \label{fig:good_vs_K}
  \end{subfigure}\hfill
  \begin{subfigure}[t]{0.48\linewidth}
    \centering
    \includegraphics[width=\linewidth]{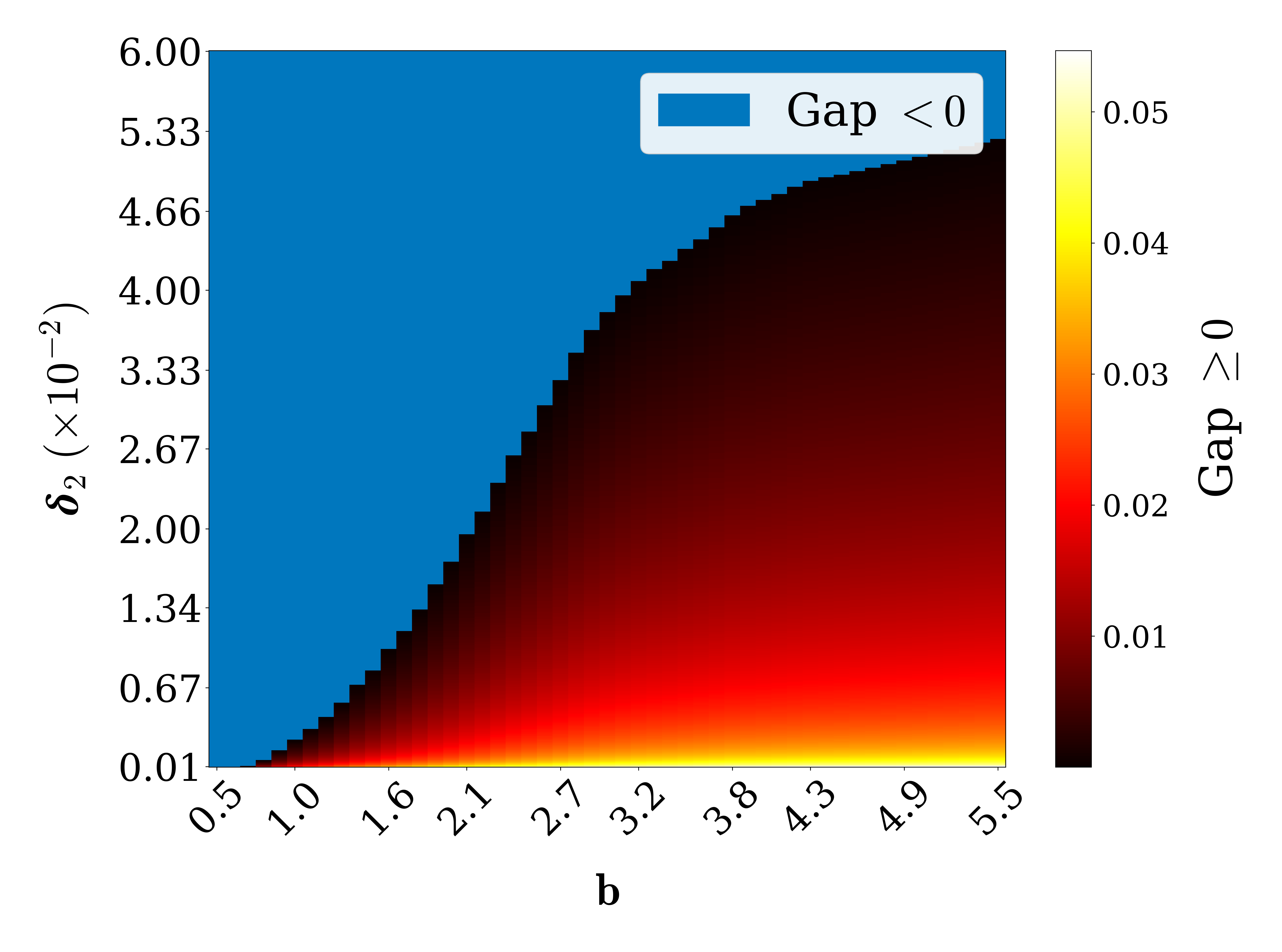}
    \caption{\textbf{Heat map of configurations \((b,\delta_{2})\) at \(K=500\) and \(\delta=0.1\)}.The best settings form a dark band at small \(\delta_{2}\) and moderate \(b\).} \label{fig:delta2_b}
  \end{subfigure}
  \caption{{Low-tuning heuristics for parameters} for 8 classifiers on BANK dataset.}
  \label{fig:abla_low_tune}
\end{figure}

\textbf{Choosing the partition size $K$:}
A triplet \((K,\delta_2,b)\) is deemed \emph{good} if the optimized lower bound differs from the oracle loss by at most \(10^{-3}\). Figure~\ref{fig:good_vs_K} shows that the number of good triplets grows sharply with \(K\) across all 8 classifiers. This aligns with sub-section~\ref{abla_partition}, where larger \(K\) improves stability of the bound when partitions are reasonably balanced.

\emph{Practical suggestion:} prefer large \(K\) (e.g., \(K\!\ge\!400\)); when feasible, set \(K\) close to the small test set size so that each data point in the small test set effectively anchors one partition. If many partitions are empty or extremely small, reduce \(K\) or increase the number of iterations.

\textbf{Heuristics for \(b\) and \(\delta_2\):}
The parameter \(b\) balances per-partition mass and prevents a single small test point from dominating the lower bound. Very large \(b\) renders the adjustment inactive, while very small \(b\) makes it hard for partitions to accumulate enough points, hurting the estimate.

 Figure~\ref{fig:delta2_b} delineates where OSYN is valid (colored, $\text{Gap}\ge0$) versus invalid (blue, $\text{Gap}<0$), with darker shades indicating tighter bounds (smaller positive $\text{Gap}$). The feasible region expands as \(b\) increases, while the tightest configurations lie along the dark ridge at small \(\delta_{2}\) and moderate \(b\). The boundary between these regions reveals a clear trade-off: increasing $b$ requires a smaller $\delta_{2}$ to remain valid, whereas smaller $b$ tolerates a larger $\delta_{2}$. Empirically most models work well with \(b\in[0.5,2]\), paired with moderate \(\delta_2\in[10^{-3}, 10^{-2}]\). A consistent exception is Random Forest, which benefits from a wider \(b\) range (often \(b\ge 4\)) together with very small \(\delta_2\) (typically \(\le 10^{-3}\)).

\subsection{Discussion}
In this section, we outline key practical considerations for applying OSYN. 

\textbf{Prerequisites.} OSYN relies on the availability of a pretrained generative model from a related domain, which can be adapted with a small amount of target data so that the generated samples approximate the true distribution. This assumption is increasingly practical given recent progress \cite{zhang2025generative,wang2018transferring,luvcic2019high} in fine-tuning generative models with limited data. We recommend on the order of several hundred labeled samples to obtain more accurate and stable estimates.
    
\textbf{Diagnostics for generator adequacy. }The effectiveness of OSYN depends critically on the quality of the generator, so basic diagnostics should be performed before use. On the qualitative side, one may compare histograms or visualize real versus generated samples. Quantitatively, measures such as FID, IS, or embedding-based distances (e.g., Wasserstein, MMD) provide stronger evidence of distributional similarity. Furthermore, unsuitable samples can be identified and filtered (for instance, using CLIP-based zero-shot validation or classifier confidencefor with image data)

\textbf{When OSYN is preferable.} Direct evaluation with a large labeled test set is always the gold standard. However, OSYN is particularly attractive when labeled data are scarce, biased, or imbalanced, and when a reliable pretrained generator is available. In such scenarios, OSYN offers an advantage by producing confidence-based guarantees on performance, rather than relying on a single point estimate from a limited evaluation set.

\textbf{Failure modes and mitigations.} OSYN may fail if the synthetic distribution deviates substantially from the true data, leading to loose bounds. This risk highlights the importance of generator adequacy checks. Another potential issue is mode collapse, where the generator misses rare but important patterns. In practice, explicit monitoring of sample diversity and the use of conditional generation techniques can help mitigate these problems and improve coverage of the target space.

\section{Conclusion}\label{sec13}

This paper introduces OSYN for evaluating model performance using synthetic data when labeled data is limited. The method is supported by both theoretical analysis and empirical validation, showing that OSYN provides reliable performance estimates. Simulations reveal a strong link between generator quality and estimate accuracy, underscoring the importance of high-fidelity synthetic data. Extensive experiments across classification and regression tasks (Appendix~\ref{app-Regression}) demonstrate OSYN's effectiveness and robustness under various data generation settings. An implementation of OSYN appears at https://github.com/daniel-nguyen121/osyn/.

 Although providing a theoretically-grounded method for model evaluation, this work remains some limitations. Firstly, OSYN only relies on the lower bound on the true error of a model, but lack an upper bound to provide a two-sided confidence interval for the error estimate. Deriving a tight but computable bound should be an interesting direction. Secondly, OSYN has some hyperparameters that often need a careful choice to obtain a good performance for OSYN. Lastly, we have not yet evaluated the method on high-dimensional or highly structured data, such as text or images, where the feature space can be extremely large and sparse. In such settings, generating synthetic samples can require higher computational costs. Addressing these challenges is an important direction for future work.

\section*{Acknowledgement} 
This research is funded by Vietnam National Foundation for Science and Technology Development (NAFOSTED) under grant number 102.01-2025.47.

\backmatter

\newpage
\begin{appendices}

\section{Proofs for Main Results} \label{Proof-main-results}

\begin{proof}[Proof of Theorem \ref{low-bound}.]
Following the proof of Theorem 7 in \cite{than2025gentle}, we decompose $E = F(P_0, \vh) =   \sum\limits_{i=1}^K p_i a_i$. Note that all $a_i$'s are fixed w.r.t. the sampling of $(g_1, ..., g_K)$. 
Since  $(g_1, ..., g_K)$ is a mutinomial random variable with parameters $g$ and $(p_1, ..., p_K)$. For   any $M \in (0, \beta/\hat{a})$, Lemma 2 in \cite{kawaguchi2022robustness} shows that

\begin{align*}
\Pr\left(\sum_{i=1}^K p_i a_i \ge \sum_{i=1}^K\frac{g_i}{g} a_i - M  \right)  
&\ge 1- \exp\left(-\frac{gM}{2\hat{a}} \min\left\{1, \frac{\hat{a}M}{\beta}\right\}\right) = 1- \exp\left(-\frac{gM^2}{2\beta}\right) 
\end{align*}

For any $\delta_1> \exp(-0.5{g\beta}/{\hat{a}^2})$, choosing $M = \sqrt{{-2\beta \ln\delta_1}/{g}}$, we obtain
$
\Pr\left(\sum_{i=1}^K p_i a_i \ge \sum_{i=1}^K \frac{g_i}{g} a_i - \sqrt{\frac{-2\beta \ln\delta_1}{g}}  \right)  \ge  1- \frac{\delta_1}{2}
$. 
In other words, the following holds with probability at least $1-\delta_1$:
\begin{eqnarray}
\label{app-eq-lower-bound-01}
E &\ge& \sum_{i=1}^K \frac{g_i}{g} a_i - \sqrt{\frac{-2\beta \ln\delta_1}{g}} =  \sum_{i \in \mT_S} \frac{g_i}{g} a_i - \sqrt{\frac{-2\beta \ln\delta_1}{g}} 
\end{eqnarray}

We next observe that $\sqrt{\beta} = \sqrt{ 2\sum\limits_{i=1}^K p_i a_i^2} \le \sqrt{ 2\hat{a}\sum\limits_{i=1}^K p_i a_i} = \sqrt{ 2\hat{a}} \sqrt{ E}$, due to $0 \le a_i \le \hat{a}$ for all $i$. Utilizing this information into (\ref{app-eq-lower-bound-01}), we have  the following  with probability at least $1-\delta_1$:
\begin{eqnarray}
\label{lower-bound-03}
E &\ge& \sum\limits_{i \in \mT_S} \dfrac{g_i}{g} a_i -  \sqrt{E} \sqrt{\frac{-4 \hat{a} \ln\delta_1}{g}}
\end{eqnarray}

It is easy to see that each $\frac{g_i}{g} F(\mS_i, \vh)$ is an empirical version of $\frac{g_i}{g} a_i$, since $\vh$ is fixed w.r.t. the sampling of $\mS$ and $\mG$. Furthermore, $0 \le \frac{g_i}{g} F(\mS_i, \vh) \le C_h {g_i}/{g}$ for each $i$. Therefore, for all $t>0$, Hoeffding's inequality implies 
\begin{equation*}
  \Pr\left[\sum\limits_{i \in \mT_S}\frac{g_i}{g} a_i -\sum\limits_{i \in \mT_S} \frac{g_i}{g}F(\mS_i, \vh) \geq t\right] \leq \exp\left(-\frac{2t^2}{C_h^2\sum\limits_{i \in \mT_S} \frac{g_i^2}{g^2}}\right)  
\end{equation*}
Choosing $\delta_2=\exp\left(-\frac{2t^2}{C_h^2\sum_{i \in \mT_S} (g_i/g)^2}\right)$ means $t = C_h\sqrt{(-0.5\ln{\delta_2}) \sum_{i \in \mT_S} (g_i/g)^2}$. 
Hence, 
\begin{equation}
\label{app-eq-lower-bound-06}
\Pr\left[\sum\limits_{i \in \mT_S}\dfrac{g_i}{g} a_i\geq \sum\limits_{i \in \mT_S}\dfrac{g_i}{g}F(\mS_i, \vh) - C_h\sqrt{(-0.5\ln{\delta_2}) \sum_{i \in \mT_S} (g_i/g)^2} \right] \geq 1-\delta_2
\end{equation}

Combining (\ref{lower-bound-03}) and (\ref{app-eq-lower-bound-06}) and the union bound, the following holds with probability at least $1-\delta_1-\delta_2$:
\begin{equation}
\label{app-eq-lower-bound-07}
E \geq \sum\limits_{i \in \mT_S}\dfrac{g_i}{g}F(\mS_i, \vh)-C_h\sqrt{\dfrac{1}{2}\ln\dfrac{1}{\delta_2}\sum\limits_{i \in \mT_S} \dfrac{g_i^2}{g^2}} - \sqrt{E} \sqrt{\frac{-4 \hat{a} \ln\delta_1}{g}}
\end{equation}

Observe further that
\begin{align}
F(\mG, \vh) - \sum\limits_{i \in \mT_S}\dfrac{g_i}{g}F(\mS_i, \vh) &= \dfrac{1}{g} \sum\limits_{i \in \mT_S}\sum\limits_{\vu \in \mG_i}l(\vh, \vu)- \sum\limits_{i \in \mT_S}\dfrac{g_i}{g} F(\mS_i, \vh) = \dfrac{1}{g} \sum\limits_{i \in \mT_S}\left[\sum\limits_{\vu \in \mG_i}l(\vh, \vu) - g_i F(\mS_i, \vh) \right]\\
\nonumber
&=\dfrac{1}{g} \sum\limits_{i \in \mT_S}\sum\limits_{\vu \in \mG_i}\left[l(\vh, \vu) - F(\mS_i, \vh)\right] = \dfrac{1}{g} \sum\limits_{i \in \mT_S}\sum\limits_{\vu \in \mG_i}\left[l(\vh, \vu)-\dfrac{1}{n_i}\sum\limits_{\vs \in \mS_i}l(\vh, \vs)\right]\\
\nonumber
&= \dfrac{1}{g} \sum\limits_{i \in \mT_S}\sum\limits_{\vu \in \mG_i}\sum\limits_{\vs \in \mS_i}\dfrac{1}{n_i}[l(\vh, \vu)-l(\vh, \vs)]\\
\label{lower-bound-04}
&\leq \dfrac{1}{g} \sum\limits_{i \in \mT_S}\sum\limits_{\vu \in \mG_i}\sum\limits_{\vs \in \mS_i}\dfrac{1}{n_i}|l(\vh, \vu)-l(\vh, \vs)|  = \sum\limits_{i \in \mT_S}\dfrac{g_i}{g}\epsilon(\mG_i, \mS_i).
\end{align}
Combining this with (\ref{app-eq-lower-bound-07}) results in 
\begin{eqnarray}
  E &\geq& F(\mG, \vh) - \sum\limits_{i \in \mT_S}\dfrac{g_i}{g}\epsilon(\mG_i, \mS_i)-C_h\sqrt{\dfrac{1}{2}\ln\dfrac{1}{\delta_2}\sum\limits_{i \in \mT_S} \dfrac{g_i^2}{g^2}}- \sqrt{E} \sqrt{\frac{-4 \hat{a} \ln\delta_1}{g}}  \\
  &\ge& F(\mG, \vh) -  \epsilon_h(\mG, \mS) - B - \sqrt{E} \sqrt{4D}
\end{eqnarray}
Solving this inequality for $\sqrt{E}$ under the condition of $F(\mG, \vh) \ge  \epsilon_h(\mG, \mS) + B$ will complete the proof.
\end{proof}

\begin{proof}[Proof of Theorem \ref{Asymtotic-theorem}.]
Following the same proof of Theorem \ref{low-bound}, we first  show the same result for $E = F(P_0, \vh)$ as (\ref{lower-bound-03}):
\begin{eqnarray}
\label{Asymtotic-01}
E &\ge& \sum_{i \in [K]} \dfrac{g_i}{g} a_i -  \sqrt{E} \sqrt{\frac{-4 \hat{a} \ln\delta_1}{g}}
\end{eqnarray}
with probability at least $1-\delta_1$, for any $\delta_1 > \exp(- {0.5 g\beta}/{\hat{a}^2})$. Using the same arguments as (\ref{lower-bound-04}), we can show that
\begin{align}
\label{Asymtotic-02}
F(P_g, \vh) - \sum_{i \in [K]}\frac{g_i}{g} a_i  
\le \sum_{i \in [K]} \frac{g_i}{g} d(P_g, P_0 | \gZ_i)
\end{align}
Combining this with (\ref{Asymtotic-01}), we have  the following  with probability at least $1-\delta_1$: 
\begin{eqnarray}
  E &\ge& F(P_g, \vh) -  \sum_{i \in [K]} \frac{g_i}{g} d(P_g, P_0 | \gZ_i) -  \sqrt{E} \sqrt{\frac{-4 \hat{a} \ln\delta_1}{g}}
\end{eqnarray}
Solving this inequality for $\sqrt{E}$, we have  the following  with probability at least $1-\delta_1$: 
\begin{equation}
\label{Asymtotic-03}
   E \ge \left(\sqrt{F(P_g, \vh) - \sum_{i \in [K]} \frac{g_i}{g} d(P_g, P_0 | \gZ_i) + \frac{- \hat{a} \ln\delta_1}{g}} - \sqrt{\frac{- \hat{a} \ln\delta_1}{g}}\right)^2
\end{equation}
As $g \rightarrow \infty$, it is easy to see that $\frac{g_i}{g} \rightarrow p_i$, $\frac{- \hat{a} \ln\delta_1}{g} \rightarrow 0$ and we can choose $\delta_1 \rightarrow 0$. Therefore $E \ge F(P_g, \vh) - \sum_{i \in [K]} p_i d(P_g, P_0 | \gZ_i)$ almost surely.

Finally, observe that $\sum_{i \in [K]} p_i^g a_i   = \sum_{i \in [K]} p_i^g a_i - F(P_g, \vh) + F(P_g, \vh)$. Using the same arguments as (\ref{lower-bound-04}), we can show that $\sum_{i \in [K]} p_i^g a_i - F(P_g, \vh) \le \sum_{i \in [K]} p_i^g d(P_g, P_0 | \gZ_i)$. Therefore, $\sum_{i \in [K]} p_i^g a_i \le F(P_g, \vh) + \sum_{i \in [K]} p_i^g d(P_g, P_0 | \gZ_i)$, completing the proof.
\end{proof}


\section{Experimental Details}
\label{exp-details}

This appendix includes details on the experiments, including (i) the datasets, and (ii) the different
settings of the experiments.

\subsection{Datasets}
\label{app-datasets}

\subsubsection{Simulated data}
\label{simulated-data}
In experiment \ref{sil-ass-qua-gen-qua}, we create simulation datasets from a Gaussian mixture distribution $P_\vx(\vx) = \sum\limits_{k \in [5]}\pi_k\mathcal{N}(\vx|\mu_k, \Sigma_k), k = \overline{1,5}$, where
$\mu_1= [0,0], \mu_2 = [12, 15], \mu_3 = [15, 6], \mu_4 = [6,7], \mu_5 = [3, 18]$; $\pi_1 = 1/20; \pi_2 = 3/20; \pi_3= 4/20, \pi_4 = 5/20, \pi_5 = 7/20$, and 
$$\Sigma_1 = \begin{bmatrix}
    2 & 0.5 \\
    0.5 & 4 
  \end{bmatrix}, 
  \Sigma_2 = \begin{bmatrix}
    5 & -2 \\
    -2 & 7 
  \end{bmatrix}, 
  \Sigma_3 = \begin{bmatrix}
    1 & 0.9 \\
    0.9 & 5 
  \end{bmatrix},
  \Sigma_4 = \begin{bmatrix}
    10 & -7 \\
    -7 & 15 
  \end{bmatrix}, 
  \Sigma_5 = \begin{bmatrix}
    5 & 0.9 \\
    0.9 & 5 
  \end{bmatrix}.$$

$D_{train}, \mS, D_{oracle}$ have the sizes of 5000, 500, and 20,000, respectively. In experiment~\ref{sil-ass-qua-gen-qua}, we randomly choose 500 points of the small test set from only one class. Meanwhile, for the classification task in experiment~\ref{sil-ass-qua-model-qua}, we select 500 samples from two classes among the total of five classes. 
 
\subsubsection{Real-world tabular datasets}
\label{real-tabular-data}

\textbf{CREDIT Dataset}~\cite{yeh2009comparisons}. The dataset contains information about credit card clients from a bank in Taiwan. The task is to classify whether a customer's payment is default or not. Among the total of 30,000 clients, 6,636 clients (22\%) are cardholders with default payment. The dataset has 23 features, including both categorical and numerical ones. 

\textbf{BANK Dataset}~\cite{moro2014data}. The Bank Marketing dataset is derived from direct marketing campaigns conducted by a Portuguese banking institution. It comprises 45,211 client records, featuring 17 input variables and one binary outcome. The goal of the prediction task is to ascertain whether a client will subscribe to a term deposit. Out of the total records, only 5,289 (11.7\%) reflect positive responses, indicating a subscription of ``yes".   

\textbf{ADULT Dataset}~\cite{adult_2}. The Adult dataset has 32,561 instances with a total of 14 features, including
demographic (age, gender, race), personal (marital status), and financial (income) features, amongst others. The classification task is to predict whether a person's income is over \$50K a year. There are total of 7841 samples with annual income exceeding \$50K (24\%).

\subsection{Experiment}
\subsubsection{Experiment \ref{sil-ass-qua-gen-qua}}
\label{app-sil-ass-qua-gen-qua}

 The number of synthetic points is $g= 50,000$. Optimization is carried out with $T=15$ iterations, each using $N=50,000$ synthetic samples to filter $(N = |\mG^t|)$. The results are shown with $\delta_1 = 0.01, \delta_2 = 0.2$, or our lower bounds have a confidence of about $80\%$.

\subsubsection{Experiment \ref{sil-ass-qua-model-qua}}
\label{app-sil-ass-qua-model-qua}

In this experiment, the value of $\delta$ is varied between datasets but is chosen so that the classifier confidence is not less than $80\%$.

\section{Additional Ablation Studies}\label{app-more-Ablation}

\subsection{OSYN in regression problems} \label{app-Regression}

{\setlength{\extrarowheight}{2pt}
\begin{table}[t]
\centering
\fontsize{7.5pt}{8.5pt}\selectfont
\caption{Results on regression task. 
}
\label{reg task}
\setlength{\tabcolsep}{3pt}   
\begin{tabular}{
    >{\raggedright\arraybackslash}p{0.8 cm}  
    l                                       
    l l l c}         
\toprule
\multicolumn{1}{c}{\textbf{Dataset}} & 
\multicolumn{1}{c}{\textbf{Model}} & 
\multicolumn{1}{c}{\textbf{OSYN}} & 
\multicolumn{1}{c}{\textbf{Bootstrap Loss}} & 
\multicolumn{1}{c}{\textbf{Syn-wo-Opt Loss}} & 
\multicolumn{1}{c}{\textbf{Oracle Loss}} \\
\midrule
\multirow{7}{*}{BIKE}
  & Ridge  & $\mathbf{0.069}${\tiny$\pm0.001$} $(+0.009)$ & $0.091 (-0.013)^{*}$ & $0.160${\tiny$\pm0.000$} $(-0.083)^{*}$ & $0.077$ \\
  & SVR  & $0.061${\tiny$\pm0.000$} $(+0.040)$ & $\mathbf{0.091} (+0.010)$ & $0.152${\tiny$\pm0.001$} $(-0.051)^{*}$ & $0.101$ \\
  & DTR & $\mathbf{0.049}${\tiny$\pm0.001$} $(+0.006)$ & 0.066 ($-0.011)^{*}$ & $0.168${\tiny$\pm0.001$} $(-0.113)^{*}$ & $0.055$ \\
  & RFR & $\mathbf{0.040}${\tiny$\pm0.005$} $(+0.002)$ & 0.057 ($-0.014)^{*}$ & $0.152${\tiny$\pm0.002$} $(-0.110)^{*}$ & $0.043$ \\
  & GBR & $\mathbf{0.054}${\tiny$\pm0.001$} $(+0.008)$ & 0.069 ($-0.006)$ & $0.146${\tiny$\pm0.000$} $(-0.084)^{*}$ & $0.062$ \\
  & KNR  & $0.058${\tiny$\pm0.000$} $(+0.029)$ & $\mathbf{0.068}$ ($+0.019)$ & $0.154${\tiny$\pm0.001$} $(-0.067)^{*}$ & $0.087$ \\
  & ANN   & $\mathbf{0.039}${\tiny$\pm0.000$} $(+0.003)$ & 0.046 ($-0.004)^{*}$ & $0.175${\tiny$\pm0.001$} $(-0.134)^{*}$ & $0.042$ \\
\midrule
\multirow{7}{*}{ADULT}
  & Ridge & $\mathbf{0.261}${\tiny$\pm0.000$} $(+0.001)$ & 0.307 ($-0.045)^{*}$ & $0.272${\tiny$\pm0.000$} $(-0.010)^{*}$ & $0.262$ \\
  & SVR  & $\mathbf{0.234}${\tiny$\pm0.000$} $(+0.013)$ & 0.292 ($-0.046)^{*}$ & $0.267${\tiny$\pm0.000$} $(-0.020)^{*}$ & $0.247$ \\
  & DTR   & $\mathbf{0.182}${\tiny$\pm0.013$} $(+0.023)$ & 0.240 ($-0.035)^{*}$ & $0.240${\tiny$\pm0.001$} $(-0.034)^{*}$ & $0.205$ \\
  & RFR  & $\mathbf{0.188}${\tiny$\pm0.002$} $(+0.015)$ & 0.249 ($-0.045)^{*}$ & $0.242${\tiny$\pm0.000$} $(-0.039)^{*}$ & $0.203$ \\
  & GBR & $\mathbf{0.200}${\tiny$\pm0.001$} $(+0.012)$ & 0.258 ($-0.046)^{*}$ & $0.245${\tiny$\pm0.000$} $(-0.033)^{*}$ & $0.212$ \\
  & KNR  & $\mathbf{0.204}${\tiny$\pm0.002$} $(+0.011)$ & 0.269 ($-0.055)^{*}$ & $0.233${\tiny$\pm0.000$} $(-0.018)^{*}$ & $0.214$ \\
  & ANN & $\mathbf{0.226}${\tiny$\pm0.000$} $(+0.020)$ & 0.280 ($-0.033)^{*}$ & $0.321${\tiny$\pm0.001$} $(-0.075)^{*}$ & $0.246$ \\
\bottomrule
\end{tabular}
\end{table}
}

In this section, we investigate the effectiveness of OSYN on the regression task. We use two datasets, BIKE (Bike Sharing~\cite{bike_sharing_275}) and ADULT~\cite{adult_2}; for the ADULT dataset, we treat the binary target variable as a numerical variable taking values 0 and 1. The datasets are preprocessed by removing irrelevant features, normalizing continuous variables, and applying one-hot encoding to categorical variables. The partitioning into training, test, and oracle sets, as well as the procedure for finding the optimal synthetic set, is carried out in the same manner as for the classification task. The used loss is Mean Absolute Error (MAE). For each dataset, we evaluate the performance of 7 regressors, namely, Ridge, Support Vector Regression (SVR), Decision Tree Regressor (DTR), Random Forest Regressor (RFR), Gradient Boosting Regressor (GBR), KNeighbors Regressor (KNR), and MLP Regressor (MLP). The small set $\mS$ has the size of 500, and is biased, with $F(\mS, \vh) \gg F(D_{\text{oracle}}, \vh)$. 

 From Table~\ref{reg task}, we observe that the OSYN method yields the closest lower bound to the true loss in the vast majority of cases. Although there are three models trained on the BIKE dataset whose losses estimated by Bootstrap are better than those obtained by OSYN and the other baselines, for the remaining models, Bootstrap produces inaccurate estimates, with estimated losses exceeding the true loss, reflecting its limitation of being highly dependent on the set $\mS$. The random synthetic method without point selection yields poor estimates with large errors. These observations are consistent with the results obtained from the simulated data and with the classification task discussed above.

\subsection{Impact of training dataset on generator quality} \label{abla_gen_2}
To further investigate robustness, we examine how the size and quality of the training data impact the generator’s performance in lower-bound estimation. Each generator is trained separately on three datasets of varying representativeness: the full oracle dataset $D_{\mathrm{oracle}}$, a training set $D_{\mathrm{train}}$, and a small test set $D_{\mathrm{small}}$ (see Section~\ref{sil-ass-qua-model-qua}). 

\begin{table}[!t]
  \centering
  \fontsize{7pt}{8.5pt}\selectfont
  \caption{
    Effect of the generator when trained from a training data of varying size. The generators are trained on $D_{\mathrm{oracle}}$, $D_{\mathrm{train}}$, and $D_{\mathrm{small}}$, respectively. OSYN demonstrates robustness across all training conditions. 
  }
  \label{table:th2_generator}
  \setlength{\tabcolsep}{2.5pt}   
  \begin{tabular}{l l l c c c c}
    \toprule
    \textbf{Model} 
      & \textbf{Generator} 
      & \textbf{Set} 
      & \textbf{OSYN} 
      & \textbf{Bootstrap Loss} 
      & \textbf{Syn-wo-Opt Loss} 
      & \textbf{Oracle Loss} \\
    \midrule
    \multirow{9}{*}{ANN}
      & \multirow{3}{*}{GC}
        & $D_{train}$  & $\mathbf{0.082}{\scriptstyle\pm0.005}$ $(+0.023)$ 
                      & $0.240$ $(-0.135)^{*}$ 
                      & $0.133{\scriptstyle\pm0.000}$ $(-0.028)$ 
                      & $0.105$ \\
      &                           & $D_{small}$  & $\mathbf{0.093}{\scriptstyle\pm0.001}$ $(+0.012)$ 
                      & $0.240$ $(-0.135)^{*}$ 
                      & $0.355{\scriptstyle\pm0.001}$ $(-0.250)^{*}$ 
                      & $0.105$ \\
      &                           & $D_{oracle}$ & $\mathbf{0.092}{\scriptstyle\pm0.000}$ $(+0.013)$ 
                      & $0.240$ $(-0.135)^{*}$ 
                      & $0.145{\scriptstyle\pm0.000}$ $(-0.040)^{*}$ 
                      & $0.105$ \\
    \cmidrule(l){2-7}
      & \multirow{3}{*}{Copula}
        & $D_{train}$  & $\mathbf{0.088}{\scriptstyle\pm0.000}$ $(+0.016)$ 
                      & $0.240$ $(-0.135)^{*}$ 
                      & $0.138{\scriptstyle\pm0.000}$ $(-0.033)^{*}$ 
                      & $0.105$ \\
      &                           & $D_{small}$  & $\mathbf{0.087}{\scriptstyle\pm0.001}$ $(+0.018)$ 
                      & $0.240$ $(-0.135)^{*}$ 
                      & $0.413{\scriptstyle\pm0.001}$ $(-0.308)^{*}$ 
                      & $0.105$ \\
      &                           & $D_{oracle}$ & $\mathbf{0.094}{\scriptstyle\pm0.000}$ $(+0.011)$ 
                      & $0.240$ $(-0.135)^{*}$ 
                      & $0.192{\scriptstyle\pm0.000}$ $(-0.087)^{*}$ 
                      & $0.105$ \\
    \cmidrule(l){2-7}
      & \multirow{3}{*}{CTGAN}
        & $D_{train}$  & $\mathbf{0.090}{\scriptstyle\pm0.000}$ $(+0.014)$ 
                      & $0.240$ $(-0.135)^{*}$ 
                      & $0.111{\scriptstyle\pm0.001}$ $(-0.007)$ 
                      & $0.105$ \\
      &                           & $D_{small}$  & $\mathbf{0.083}{\scriptstyle\pm0.002}$ $(+0.016)$ 
                      & $0.240$ $(-0.135)^{*}$ 
                      & $0.465{\scriptstyle\pm0.001}$ $(-0.361)^{*}$ 
                      & $0.105$ \\
      &                           & $D_{oracle}$ & $\mathbf{0.095}{\scriptstyle\pm0.000}$ $(+0.010)$ 
                      & $0.240$ $(-0.135)^{*}$ 
                      & $0.141{\scriptstyle\pm0.000}$ $(-0.036)^{*}$ 
                      & $0.105$ \\
    \midrule
    \multirow{9}{*}{GB}
      & \multirow{3}{*}{GC}
        & $D_{train}$  & $\mathbf{0.068}{\scriptstyle\pm0.000}$ $(+0.028)$ 
                      & $0.208$ $(-0.111)^{*}$ 
                      & $0.152{\scriptstyle\pm0.001}$ $(-0.055)^{*}$ 
                      & $0.097$ \\
      &                           & $D_{small}$  & $\mathbf{0.082}{\scriptstyle\pm0.000}$ $(+0.015)$ 
                      & $0.208$ $(-0.111)^{*}$ 
                      & $0.352{\scriptstyle\pm0.000}$ $(-0.255)^{*}$ 
                      & $0.097$ \\
      &                           & $D_{oracle}$ & $\mathbf{0.080}{\scriptstyle\pm0.000}$ $(+0.016)$ 
                      & $0.208$ $(-0.111)^{*}$ 
                      & $0.181{\scriptstyle\pm0.001}$ $(-0.084)^{*}$ 
                      & $0.097$ \\
    \cmidrule(l){2-7}
      & \multirow{3}{*}{Copula}
        & $D_{train}$  & $\mathbf{0.074}{\scriptstyle\pm0.000}$ $(+0.023)$ 
                      & $0.208$ $(-0.111)^{*}$ 
                      & $0.139{\scriptstyle\pm0.001}$ $(-0.043)^{*}$ 
                      & $0.097$ \\
      &                           & $D_{small}$  & $\mathbf{0.070}{\scriptstyle\pm0.000}$ $(+0.027)$ 
                      & $0.208$ $(-0.111)^{*}$ 
                      & $0.419{\scriptstyle\pm0.001}$ $(-0.322)^{*}$ 
                      & $0.097$ \\
      &                           & $D_{oracle}$ & $\mathbf{0.084}{\scriptstyle\pm0.000}$ $(+0.012)$ 
                      & $0.208$ $(-0.111)^{*}$ 
                      & $0.195{\scriptstyle\pm0.000}$ $(-0.098)^{*}$ 
                      & $0.097$ \\
    \cmidrule(l){2-7}
      & \multirow{3}{*}{CTGAN}
        & $D_{train}$  & $\mathbf{0.080}{\scriptstyle\pm0.000}$ $(+0.017)$ 
                      & $0.208$ $(-0.111)^{*}$ 
                      & $0.116{\scriptstyle\pm0.000}$ $(-0.019)^{*}$ 
                      & $0.097$ \\
      &                           & $D_{small}$  & $\mathbf{0.078}{\scriptstyle\pm0.000}$ $(+0.019)$ 
                      & $0.208$ $(-0.111)^{*}$ 
                      & $0.483{\scriptstyle\pm0.000}$ $(-0.386)^{*}$ 
                      & $0.097$ \\
      &                           & $D_{oracle}$ & $\mathbf{0.084}{\scriptstyle\pm0.000}$ $(+0.013)$ 
                      & $0.208$ $(-0.111)^{*}$ 
                      & $0.143{\scriptstyle\pm0.000}$ $(-0.047)^{*}$ 
                      & $0.097$ \\
    \midrule
    \multirow{9}{*}{LDA}
      & \multirow{3}{*}{GC}
        & $D_{train}$  & $\mathbf{0.071}{\scriptstyle\pm0.000}$ $(+0.038)$ 
                      & $0.228$ $(-0.119)^{*}$ 
                      & $0.143{\scriptstyle\pm0.000}$ $(-0.034)$ 
                      & $0.109$ \\
      &                           & $D_{small}$  & $\mathbf{0.086}{\scriptstyle\pm0.000}$ $(+0.023)$ 
                      & $0.228$ $(-0.119)^{*}$ 
                      & $0.373{\scriptstyle\pm0.000}$ $(-0.264)^{*}$ 
                      & $0.109$ \\
      &                           & $D_{oracle}$ & $\mathbf{0.085}{\scriptstyle\pm0.000}$ $(+0.024)$ 
                      & $0.228$ $(-0.119)^{*}$ 
                      & $0.158{\scriptstyle\pm0.001}$ $(-0.049)^{*}$ 
                      & $0.109$ \\
    \cmidrule(l){2-7}
      & \multirow{3}{*}{Copula}
        & $D_{train}$  & $\mathbf{0.084}{\scriptstyle\pm0.002}$ $(+0.025)$ 
                      & $0.228$ $(-0.119)^{*}$ 
                      & $0.148{\scriptstyle\pm0.000}$ $(-0.039)^{*}$ 
                      & $0.109$ \\
      &                           & $D_{small}$  & $\mathbf{0.085}{\scriptstyle\pm0.003}$ $(+0.024)$ 
                      & $0.228$ $(-0.119)^{*}$ 
                      & $0.419{\scriptstyle\pm0.000}$ $(-0.310)^{*}$ 
                      & $0.109$ \\
      &                           & $D_{oracle}$ & $\mathbf{0.089}{\scriptstyle\pm0.000}$ $(+0.020)$ 
                      & $0.228$ $(-0.119)^{*}$ 
                      & $0.200{\scriptstyle\pm0.000}$ $(-0.092)^{*}$ 
                      & $0.109$ \\
    \cmidrule(l){2-7}
      & \multirow{3}{*}{CTGAN}
        & $D_{train}$  & $\mathbf{0.089}{\scriptstyle\pm0.000}$ $(+0.020)$ 
                      & $0.228$ $(-0.119)^{*}$ 
                      & $0.146{\scriptstyle\pm0.001}$ $(-0.037)^{*}$ 
                      & $0.109$ \\
      &                           & $D_{small}$  & $\mathbf{0.085}{\scriptstyle\pm0.000}$ $(+0.024)$ 
                      & $0.228$ $(-0.119)^{*}$ 
                      & $0.474{\scriptstyle\pm0.000}$ $(-0.365)^{*}$ 
                      & $0.109$ \\
      &                           & $D_{oracle}$ & $\mathbf{0.090}{\scriptstyle\pm0.001}$ $(+0.019)$ 
                      & $0.228$ $(-0.119)^{*}$ 
                      & $0.121{\scriptstyle\pm0.000}$ $(-0.012)$ 
                      & $0.109$ \\
    \bottomrule
  \end{tabular}
\end{table}

 Table~\ref{table:th2_generator} presents the results. Despite the expected degradation in generator quality from $D_{\mathrm{oracle}}$ to $D_{\mathrm{small}}$, OSYN maintains stable performance across all settings. Notably, for Gaussian Copula, the lower-bound value estimated using $D_{\mathrm{small}}$ is even closer to the true loss than when trained on larger datasets. This indicates that OSYN can compensate for reduced generator performance, further supporting its robustness.

\begin{figure}[!h]
  \centering
  \begin{subfigure}[t]{0.32\textwidth}
    \centering
    \includegraphics[width=\linewidth]{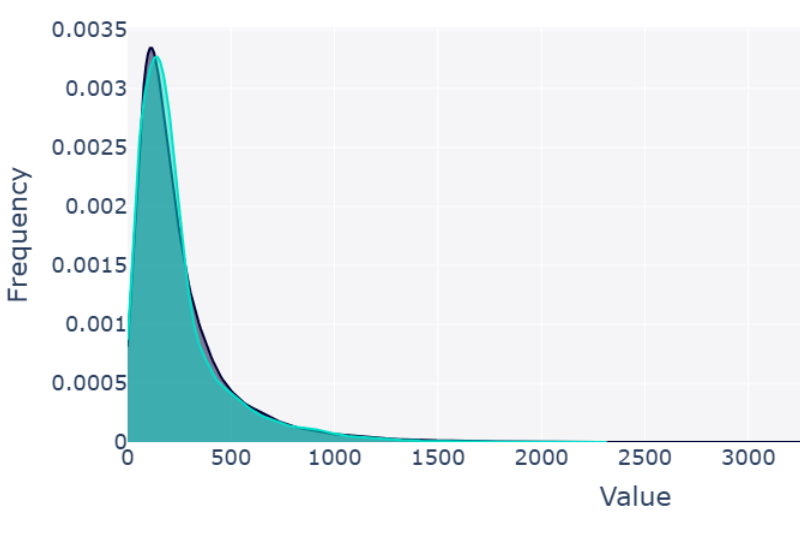} 
    \caption{$D_{\mathrm{train}}$}
  \end{subfigure}\hfill
  \begin{subfigure}[t]{0.32\textwidth}
    \centering
    \includegraphics[width=\linewidth]{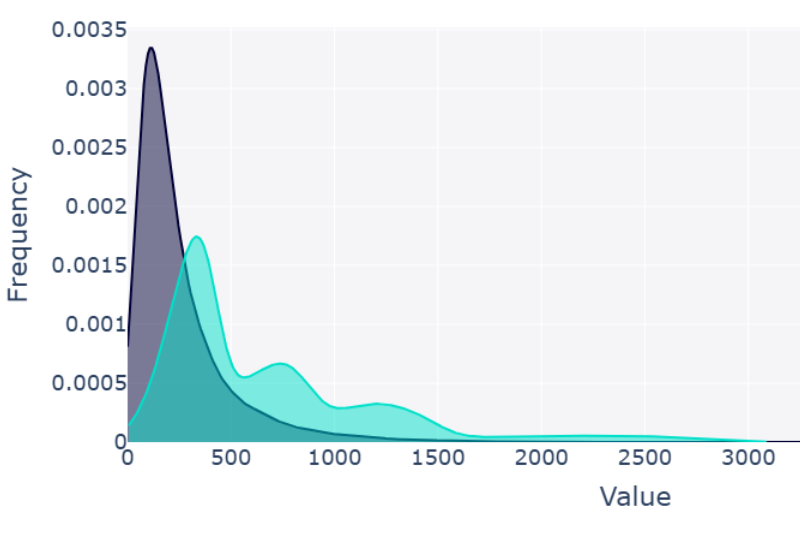} 
    \caption{$D_{\mathrm{small}}$}
  \end{subfigure}\hfill
  \begin{subfigure}[t]{0.32\textwidth}
    \centering
    \includegraphics[width=\linewidth]{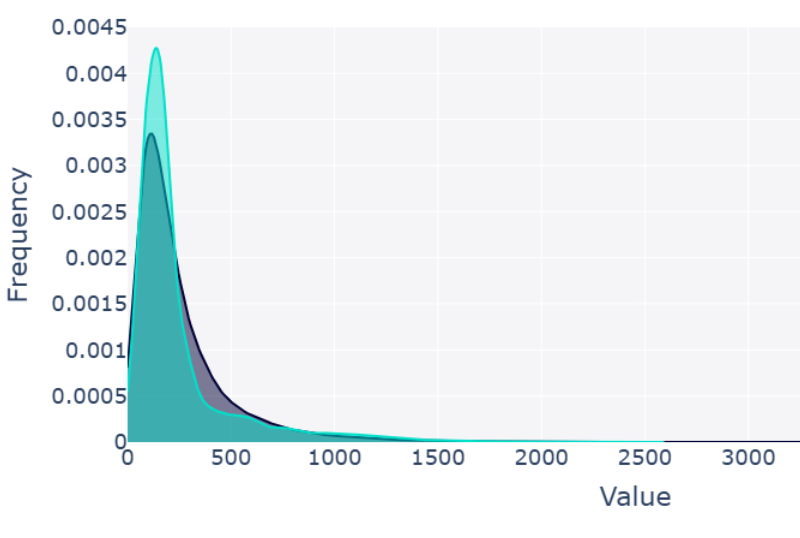} 
    \caption{$D_{\mathrm{oracle}}$}
  \end{subfigure}
  \caption{Distribution of the synthetic data (cyan) from CTGAN with different training conditions vs.\ the oracle data (gray) for the most important feature \emph{``duration''}.}
  \label{fig: gb_explain}
\end{figure}

 Figure~\ref{fig: gb_explain} provides a qualitative illustration of the data distribution learned by CTGAN across the three training sets. The CTGAN model trained on $D_{\mathrm{train}}$ best replicates the oracle distribution, while the model trained on $D_{\mathrm{small}}$ exhibits a clear shift. Nonetheless, the corresponding lower bound values under OSYN remain closely aligned.

\subsection{Computational cost}

This section presents the computational cost of OSYN. We evaluate the cost using three metrics:
\begin{enumerate}
    \item Wall-clock time (min): the total elapsed time of the algorithm;
    \item Peak CPU memory (GB): the maximum RAM usage during execution, recorded using the \texttt{psutil} library;
    \item Peak GPU memory (GB): the maximum VRAM usage during execution, recorded via the NVIDIA Management Library (\texttt{pynvml}).
\end{enumerate}

\begin{table}[tp]
\centering
\caption{Computational cost of OSYN compared to baselines. }
\label{tab:comp_cost}
\begin{tabular}{l l r c c}
\toprule
\textbf{K} & \textbf{Method} & \textbf{Wall time (min)} & \textbf{Peak CPU (GB)} & \textbf{Peak GPU (GB)} \\
\midrule
500  & OSYN & 180.416 & 1.536 & 1.926 \\
     & Bootstrap  & 0.082   & 0.738 & 0.361 \\
     & Syn-Wo-Opt  & 0.024   & 0.988 & 0.418 \\
\midrule
1000 & OSYN & 219.157 & 1.490 & 1.926 \\
     & Bootstrap & 0.095   & 0.745 & 0.361 \\
     & Syn-Wo-Opt  & 0.022   & 0.976 & 0.418 \\
\midrule
5000 & OSYN & 170.252 & 1.998 & 1.965 \\
     & Bootstrap  & 0.118   & 0.745 & 0.361 \\
     & Syn-Wo-Opt  & 0.026   & 0.976 & 0.418 \\
\bottomrule
\end{tabular}
\end{table}

 All runs use the ADULT dataset with an SVM classifier on a single NVIDIA T4 GPU. For OSYN and baselines,  the measurement was performed with the same settings in Subsection~\ref{sil-ass-qua-model-qua}. We vary the size of set $\mS$ from 500, 1000, to 5000 to simulate small, medium, and large datasets (Table~\ref{tab:comp_cost}). The results show that OSYN incurs considerably higher computational cost—especially in wall-clock time—than Bootstrap and Syn-wo-Opt. In this experiment, GPU usage in OSYN mainly comes from partitioning synthetic points with FAISS~\cite{douze2024faiss}, while the optimization step remains CPU-bound because scikit-learn classifiers (e.g., SVM, logistic regression) do not support GPU acceleration. Here, as $K$ grows, GPU memory increases only marginally, but with higher-dimensional data and deep learning models, GPU cost would become much more significant. This underscores the need to balance computational resources and performance gains when applying OSYN to large-scale tasks.
\end{appendices}

\bibliography{sn-bibliography}

\end{document}